\documentclass{article}

\usepackage[preprint]{neurips_2022}

\usepackage[utf8]{inputenc} 
\usepackage[T1]{fontenc}    
\usepackage{hyperref}       
\usepackage{url}            
\usepackage{booktabs}       
\usepackage{amsfonts}       
\usepackage{nicefrac}       
\usepackage{microtype}      
\usepackage{xcolor}         

\usepackage{amsmath,amsthm,amssymb}

\newtheorem{theorem}{Theorem}
\newtheorem{lemma}[theorem]{Lemma}

\newtheorem{corollary}[theorem]{Corollary}
\newtheorem{definition}[theorem]{Definition}

\usepackage{thmtools,thm-restate}

\usepackage[normalem]{ulem}
\usepackage{enumitem}
\usepackage{dsfont}
\usepackage{pgfplots}
\pgfplotsset{compat=1.17}

\usepackage{tikz}
\usepackage{pgfplots}
\usetikzlibrary{calc,arrows,shapes,positioning}
\newcommand\ellipsebyfoci[4]{
  \path[#1] let \p1=(#2), \p2=(#3), \p3=($(\p1)!.5!(\p2)$)
  in \pgfextra{
    \pgfmathsetmacro{\angle}{atan2(\y2-\y1,\x2-\x1)}
    \pgfmathsetmacro{\focal}{veclen(\x2-\x1,\y2-\y1)/2/1cm}
    \pgfmathsetmacro{\lentotcm}{\focal*2*#4}
    \pgfmathsetmacro{\axeone}{(\lentotcm - 2 * \focal)/2+\focal}
    \pgfmathsetmacro{\axetwo}{sqrt((\lentotcm/2)*(\lentotcm/2)-\focal*\focal}
  }
  (\p3) ellipse[x radius=\axeone cm,y radius=\axetwo cm, rotate=\angle];
}
\tikzstyle{point} = [circle, minimum width=3.5pt, fill, inner sep=0pt]
\tikzstyle{bluepoint} = [circle, minimum width=3.5pt, draw=blue!60!white, fill=blue!50!white, inner sep=0pt]

\usepackage[algo2e,ruled]{algorithm2e}

\newtheorem{claim}{Claim}

\newcommand{\E}{\mathbb{E}}

\newcommand{\R}{\mathbb{R}}
\newcommand{\Q}{\mathbb{Q}}
\newcommand{\N}{\mathbb{N}}
\newcommand{\Z}{\mathbb{Z}}
\newcommand{\norm}[2]{\|#2\|_{#1}}
\newcommand{\eps}{\epsilon}
\newcommand{\EL}{\mathcal{E}}
\newcommand{\scH}{\mathcal{H}}
\newcommand{\scQ}{\mathcal{Q}}

\newcommand{\poly}{\operatorname{poly}}

\newcommand{\lab}{\textsc{label}}
\newcommand{\seed}{\textsc{seed}}

\newcommand{\nil}{\textsc{nil}}
\newcommand{\conv}{\operatorname{conv}}

\newcommand{\C}{\mathcal{C}}
\newcommand{\scO}{\mathcal{O}}

\newcommand{\Hyp}{\scH}
\newcommand{\dotp}[1]{\left\langle{#1}\right\rangle}
\newcommand{\tp}[1]{{#1}^{\intercal}}

\newcommand{\diam}{\phi}
\newcommand{\MVE}{\ensuremath{\operatorname{MVE}}}

\newcommand{\AlgoSC}{\ensuremath{\operatorname{BinLearn}}}
\newcommand{\AlgoSCK}{\ensuremath{\operatorname{KClassLearn}}}
\newcommand{\AlgoRound}{\ensuremath{\operatorname{Round}}}
\newcommand{\CuttingPlanes}{\ensuremath{\operatorname{CPLearn}}}
\newcommand{\BallSearch}{\ensuremath{\operatorname{BallSearch}}}

\newcommand{\ev}{\mathcal{E}}
\DeclareMathOperator{\vol}{vol}

\newcommand{\lIfElse}[3]{\lIf{#1}{#2 \textbf{else}~#3}}
\newcommand{\euc}{\operatorname{euc}}
\newcommand{\lo}{\operatorname{lo}}
\newcommand{\hi}{\operatorname{hi}}

\DeclareMathOperator*{\argmax}{arg\,max}

\title{Active Learning of Classifiers \\ with Label and Seed Queries}
\author{Marco Bressan
\\
Dept.\ of CS, Univ.\ of Milan, Italy
\\
marco.bressan@unimi.it
\And
Nicolò Cesa-Bianchi
\\
DSRC \& Dept.\ of CS, Univ.\ of Milan, Italy
\\
nicolo.cesa-bianchi@unimi.it
\AND
Silvio Lattanzi
\\ 
Google
\\
silviol@google.com
\And
Andrea Paudice
\\
Dept.\ of CS, Univ.\ of Milan, Italy \&\\Istituto Italiano di Tecnologia, Italy
\\
andrea.paudice@unimi.it
\And
Maximilian Thiessen
\\
ML Research Unit, TU Wien, Austria
\\
maximilian.thiessen@tuwien.ac.at
}

\begin{document}

\maketitle

\begin{abstract}
We study exact active learning of binary and multiclass classifiers with margin. Given an $n$-point set $X \subset \R^m$, we want to learn any unknown classifier on $X$ whose classes have finite \emph{strong convex hull margin}, a new notion extending the SVM margin.
In the standard active learning setting, where only \emph{label} queries are allowed, learning a classifier with strong convex hull margin $\gamma$ requires in the worst case $\Omega\big(1+\frac{1}{\gamma}\big)^{\!\frac{m-1}{2}}$ queries. 
On the other hand, using the more powerful \emph{seed} queries (a variant of equivalence queries), the target classifier could be learned in $\scO(m \log n)$ queries via Littlestone's Halving algorithm; however, Halving is computationally inefficient.
In this work we show that, by carefully combining the two types of queries, a binary classifier can be learned in time $\poly(n+m)$ using only $\scO(m^2 \log n)$ label queries and $\scO\big(m \log \frac{m}{\gamma}\big)$ seed queries; the result extends to $k$-class classifiers at the price of a $k!k^2$ multiplicative overhead. Similar results hold when the input points have bounded bit complexity, or when only one class has strong convex hull margin against the rest. We complement the upper bounds by showing that in the worst case any algorithm needs $\Omega\big(k m \log \frac{1}{\gamma}\big)$ seed and label queries to learn a $k$-class classifier with strong convex hull margin $\gamma$. 
\end{abstract}

\section{Introduction}
\label{sec:intro}
This work investigates efficient algorithms for exact active learning of binary and multiclass classifiers in the transductive setting. Given a set $X$ of $n$ points in $\R^m$, our goal is to learn a function $h : X \to [k]$ belonging to some class $\scH$. In the classic active learning framework, $h$ identifies a subset of $X$, and the algorithm learns $h$ via queries \lab$(x)$ that return $h(x)$ for any given $x \in X$. In that case, it is well-known that $h$ can be learned with $\scO(\log n)$ \lab\ queries if the \emph{star number} of $\scH$ is finite~\citep{hanneke2015minimax}. Unfortunately, even simple families such as linear classifiers have unbounded star number, in which case $\Omega(n)$ \lab\ queries are needed in the worst case. To bypass this lower bound, it has become increasingly common to introduce \emph{enriched queries}, that reveal additional information on $h$ and are plausible in practice. One notable example is that of \emph{comparison} queries for linear separators in $\R^m$ which, given any pair of points $x,y \in X$, reveal which one is closer to the decision boundary. As proven by~\cite{Kane17}, under some margin assumptions the combination of \lab\ and comparisons yields exponential savings, allowing one to learn linear separators with only $\scO(\log n)$ queries.

In this work we combine \lab\ queries with \emph{seed queries}. For any $U \subset X$ and any $i \in [k]$, a query \seed$(U,i)$ returns an abitrary point $x$ in $U \cap C_i$, where $C_i = h^{-1}(i)$, or $\nil$ if no such $x$ exists. \seed\ queries are natural in certain settings like crowdsourcing---e.g., finding the image of a car, see also~\cite{Beygelzimer16-seedqueries}---and have been used implicitly or explicitly in several works~\citep{HannekePHD,BalcanHanneke12,Attenberg2010,Tong01-mistake-queries,Doyle2011,BCLP21-density}. It is not hard to see that, using \seed\ alone, one can implement Littlestone's Halving algorithm and learn any $h \in \scH$ with $\scO(\log |\scH|)$ queries\footnote{Halving uses equivalence queries (testing if a given subset of $X$ coincides with the target concept) each of which can be simulated using two \seed\ queries.}. For instance, linear separators in $\R^m$ can be learned with $\scO(m \log n)$ \seed\ queries.  The catch is that, save for special cases, it is not known how to run the Halving algorithm in polynomial time. Therefore, using \seed\ to obtain a computationally efficient active learning algorithm is less trivial than it seems at first glance.

The goal of this work is understanding whether one can actively learn binary and multiclass classifiers efficiently by using \lab\ and \seed\ queries together. In line with~\cite{Kane17} and other previous works, we make assumptions on $\scH$. Our main assumption is that every class $C_i$ has \emph{strong convex hull margin} $\gamma > 0$. This means that, for any $j \neq i$, $C_i$ and $C_j$ are linearly separable with a margin that is at least $\frac{\gamma}{2}$ times the diameter of $C_i$. Moreover, it is sufficient that this hold under some pseudometric $d_i$, unknown to the learner, that is homogeneous and invariant under translation (i.e., induced by a seminorm). This gives to every class its own personalized notion of distance that can be sensitive to the ``scale'' of the class. This assumption strictly generalizes the classical SVM margin; 
and, when suitably generalized, it captures stability properties of center-based clusterings \cite{Awashti2012-stability, Bilu2012-stable}.

Using \lab\ alone, \cite{BCLP21} showed that learning a multiclass classifier with (strong) convex hull margin $\gamma>0$ requires between $\Omega\big(1+\frac{1}{\gamma}\big)^{(m-1)/2}$ and $\tilde{\scO}\big(k^3 m^5\big(1+\frac{1}{\gamma}\big)^m \log n\big)$ queries. This exponential dependence on $m$ implies that, unless $m \ll \log n /\log \frac{1}{\gamma}$, one needs $\Theta(n)$ \lab\ queries in the worst case.
On the other hand our margin implies linear separability and thus, as noted above, a $\scO(m \log n)$ \seed\ query bound for the binary case, but with a running time that can be superpolynomial.
This leaves open the following problem, which is the subject of this work: \begin{quote}
Can one learn a multiclass classifier $h$ with strong convex hull margin $\gamma > 0$ on $X\!\subset\!\R^m$ in time $\poly(n\!+\!m)$ using a number of queries that grows \emph{polynomially} with $m$?
\end{quote}
We solve the above question in the affirmative by proving that, with a careful combination of \lab\ and \seed\ queries, one can do much better than using either query in isolation.
For binary classification ($k=2$), we show:
\begin{theorem}\label{thm:upper_binary}
Any binary classifier $h$ with strong convex hull margin $\gamma\!>\!0$ over $X\!\subset \R^m$ can be learned in time $\poly(n\!+\!m)$ using in expectation $\scO(m^2 \log n)$ \lab\ queries and $\scO\big(m \log \frac{m}{\gamma}\big)$ \seed\ queries.\footnote{This running time as well as those of Theorem~\ref{thm:upper_k} and~\ref{thm:oneside} are actually in high probability as implied by Theorem~\ref{thm:cp}; we have omitted this fact to keep the statements light.}
\end{theorem}
Note that, unless $\gamma$ is exceedingly small, Theorem~\ref{thm:upper_binary} uses far fewer \seed\ than \lab\ queries, which is a strength since \seed\ is arguably more expensive to implement. For instance, if $\gamma=\Omega(1/\poly(m))$ then we use $\scO(m^2 \log n)$ \lab\ queries but only $\scO(m \log m)$ \seed\ queries.
To prove Theorem~\ref{thm:upper_binary} we design a novel algorithm that works in  two phases. The first phase learns what we call an \emph{$\alpha$-rounding} of $X$ w.r.t.\ $h$. Loosely speaking, this is a partition $(X_1,X_2)$ of $X$ such that each $X_i$ lies inside $\alpha \conv(C_i)$ where $\conv(C_i)$ is the convex hull of $C_i$ (see below for the formal definition). 
We show that, in polynomial time and using $\scO(m^2 \log n)$ \lab\ queries, one can compute an $\alpha$-rounding of $X_i$ for $\alpha = \scO(m^3)$. This allows us to put $X_i$ in near-isotropic position so that $X_i$ has radius $1$ and to separate $C_1 \cap X_i$ from $C_2 \cap X_i$ with margin $\eta = \Omega(\gamma/ m^3)$. In the second phase, the algorithm uses \seed\ to implement a cutting plane algorithm that learns $C_1 \cap X_i$ and $C_2 \cap X_i$ using $\scO\big(m \log \frac{1}{\eta}\big) = \scO\big(m \log \frac{m}{\gamma}\big)$ queries in time $\poly(n+m)$. 

Using a recursive approach, Theorem~\ref{thm:upper_binary} can be extended to $k>2$ at the price of a $k!k^2$ multiplicative overhead:
\begin{theorem}\label{thm:upper_k}
Any $k$-class classifier $h$ with strong convex hull margin $\gamma>0$ over $X\!\subset \R^m$ can be learned in time $\poly(n\,+\,m)$ using in expectation $\scO(k!~k^2\, m^2 \log n)$ \lab\ queries and $\scO\big(k!~k^2\, m \log \frac{m}{\gamma}\big)$ \seed\ queries.
\end{theorem}

We also consider the case where only one class has strong convex hull margin against the rest of the points w.r.t.\ a metric $d$ induced by a norm $\norm{d}{\cdot}$. In this case we obtain a bound parameterized by the distortion $\kappa_d$ of $d$ (see Section~\ref{sub:prelim}):
\begin{theorem}\label{thm:oneside}
Suppose $C \subset X$ has strong convex hull margin $\gamma \in (0,1]$ w.r.t.\ a metric $d$ with distortion $\kappa_d < \infty$. Given only $X$, one can learn $C$ in time $\poly(n+m)$ using $\scO(\log n)$ \lab\ queries and $\scO\big(m \log \frac{\kappa_d}{\gamma}\big)$ \seed\ queries in expectation.
\end{theorem}

As an application of our cutting-plane algorithm we also show that one can learn a $k$-class classifier whose classes are pairwise linearly separable in time $\poly(n+m)$ using, in expectation, $\scO(k^2 m^3 B)$ \seed\ queries if every $x\in X$ has rational coordinates that can be encoded in $B$ bits, and $\scO(k^2 m(B + m\log m))$ \seed\ queries if every $x \in X$ lies on the grid over $[-1,1]^m$ with stepsize $2^{-B/m}$. 
It should be noted that, unlike most previous algorithms, all our algorithms do not need knowledge of $\gamma$. Moreover, all the bounds above can be turned from expectation to high probability.\footnote{Formally, for some universal constant $a > 0$, each one of our bounds in the form $\E[Q] \le q$, where $Q$ is the number of queries, implies $\Pr(Q \ge q+\epsilon q) \le \exp( - a \epsilon q)$ for all $\epsilon \ge 0$.}

Finally, we show that the algorithms of Theorem~\ref{thm:upper_binary} and~\ref{thm:upper_k} are nearly optimal:
\begin{restatable}{retheorem}{lowerbound}\label{thm:lb}
For all $m \ge 2$, all $k\geq 2$, and all $\gamma \le m^{-3/2}/16$ there exists a distribution of instances with $k$ classes in $\R^m$ with strong convex hull margin $\gamma$ where any randomized algorithm using \seed\ and \lab\ queries that returns $\C$ with probability at least $\frac{1}{2}$ makes at least $\left\lfloor\frac{k}{2}\right\rfloor\frac{m}{24} \log\frac{1}{2\gamma}$ total queries in expectation.
\end{restatable}


\subsection{Preliminaries and notation}\label{sub:prelim}
The input to our problem is a pair $(X,k)$, where $X \subset \R^m$ and $k \in \N$ with $2 \le k \le n = |X|$. The algorithm has access to oracles $O_{\lab}$ and $O_{\seed}$ which provide respectively \lab\ and \seed\ queries. The oracles $O_{\lab},O_{\seed}$ behave consistently with some target classifier $h : X \to [k]$. For any $x \in X$, $\lab(x)$ returns $h(x)$. For any $U \subseteq X$ and any $i \in [k]$, $\seed(U,i)$ returns an abitrary element $x \in U \cap C_i$ if $U \cap C_i \ne \emptyset$, and $\nil$ otherwise, where $C_i=h^{-1}(i)$. We often think of $h$ as of the partition $\C=(C_1,\ldots,C_k)$ and we call each $C_i$ a \emph{class} or \emph{cluster}.

A pseudometric is a symmetric and subadditive function $d : \R^m \times \R^m \to \R_{\ge 0}$ such that $d(x,x)=0$ for all $x \in \R^m$; unlike a metric, $d(x,y)$ need not be $0$ for $x \ne y$. In this work $d$ is always induced by a seminorm and thus homogeneous and invariant under translation: $d(u+ax,u+ay)=|a|\,d(x,y)$ for all $x,y,u \in \R^m$ and all $a \in \R$. For a pseudometric $d$ and a set $A \subset \R^m$, we let $\diam_d(A)=\sup\{d(x,y) : x,y \in A\}$ denote the diameter of $A$ under $d$. For $x \in \R^m$ and $r \ge 0$ we denote by $B_d^m(x,r)$ and $S_d^{m-1}(x,r)$ respectively the closed ball and the hypersphere with center $x$ and radius $r$ in $\R^m$ under $d$. When $d$ is omitted we assume $d=d_{\euc}$ where $d_{\euc}$ is the Euclidean metric. We may also omit the superscript if clear from the context.
The distortion of a (pseudometric) $d$ is $\kappa_d=\sup_{u,v \in S^{m-1}(0,1)} \norm{d}{u}/\norm{d}{v}$.

For any set $A \subset \R^m$, any $\mu \in \R^m$, and any $\lambda > 0$, let $\sigma(A,\mu,\lambda)=\mu+\lambda(A-\mu)$ be the scaling of $A$ about $\mu$ by a factor of $\lambda$. For two sets $A,B \subset \R^m$, we write $A \le \lambda B$ if $A \subseteq \sigma(B,z,\lambda)$ for some $z \in \R^m$. We may use $x$ in place of $A$ if $A=\{x\}$.
If $A$ is bounded, then $\MVE(A)$ denotes the minimum-volume enclosing ellipsoid (MVEE, or L\"owner-John ellipsoid) of $A$. Our proofs repeatedly use John's theorem; that is, $\sigma(E,\mu,1/m) \subseteq \conv(A)$ where $\mu$ is the center of $E = \MVE(A)$ and $\conv(A)$ is the convex hull of $A$. Given $A,B \subseteq \R^m$, we say that $A$ and $B$ are linearly separable with margin $r$ if there exist $u \in S^{m-1}(0,1)$ and $b \in \R$ such that $\dotp{u,x}+b \le -r$ for all $x \in A$ and $\dotp{u,x}+b \ge r$ for all $x \in B$.

We consider classifiers satisfying the following property:\footnote{Actually, all our upper bounds hold under a weaker condition: that for every $i$ and every $j \in [k] \setminus \{i\}$ there is a $d_{ij}$ giving the margin.} 
\begin{definition}\label{def:strong_margin}
A class $C_i$ has strong convex hull margin $\gamma > 0$ if there exists a pseudometric $d_i$ induced by a seminorm over $\R^m$ such that
    $d_i(\conv(C_j), \conv(C_i)) > \gamma \, \diam_{d_i}(C_i)$ for all $j \in [k] \setminus \{i\}$.
If this holds for all $i \in [k]$ then we say $\C$ has strong convex hull margin $\gamma$.
\end{definition}
\textbf{Remarks.} The margin of Definition~\ref{def:strong_margin} captures natural scenarios that SVM margin does not.
For instance, suppose we are clustering fruits on the basis of weight and colour. First, a fruit weighting more than, say, $1.5$ times the typical weight of a species probably does not belong to it; but the typical weight varies greatly across species. Our margin captures this scenario, as it is expressed as a fraction of the class' diameter. Second, different fruit species have different separating features; for instance, weight does not separate well oranges from bananas, but colour does. Our margin captures this aspect, too, by allowing the metric that determines the margin to be a function the class.
It is also known that the SVM margin $\gamma_{\text{SVM}}$ can be arbitrarily smaller than $\gamma$; for instance there are simple cases with $\gamma > 1$ but $\gamma_{\text{SVM}} < e^{-n}$ (see~\cite{BCLP21}). Hence a large $\gamma$ does not imply good bounds for standard algorithms based on SVM margin (e.g., the Perceptron).

\section{Related work}\label{sec:related}
It is well known that active learning may achieve exponential savings in label complexity. That is, there are natural concept classes that can be learned with a number of \lab\ queries exponentially smaller than that of passive learning. \cite{hanneke2015minimax} characterize the label complexity of concept classes in terms of their star number. However, the star number of many natural classes such as linear classifiers is unbounded, implying a strong lower bound of $\Omega(n)$ \lab\ queries.

This and other negative results motivated research on enriched queries. \cite{Kane17} prove that active learnability is characterized by the inference dimension of the concept class $\scH$ under the set of allowed queries $\scQ$, as long as those queries are local (i.e., are a function of a constant number of instances). This yields exponential savings when $\scH$ is the class of linear separators and $\scQ$ contains label queries and comparison queries (which, given two points, reveal which one is closer to the decision boundary), provided the classes have SVM margin or bounded bit complexity. \cite{hopkins2020power} give similar results under distributional assumptions. Unfortunately, bounded inference dimension does not automatically yield efficient algorithms, although it implies active learning algorithms with bounded memory~\citep{Hopkins21boundedmem}. 

\seed\ and their variants are motivated and used by~\citet{HannekePHD} as \textsl{positive example queries}, by~\citet{BalcanHanneke12} as \textsl{conditional class queries}, and by~\citet{Beygelzimer16-seedqueries,Attenberg2010} as \textsl{search queries}. They are also used implicitly by~\citet{Tong01-mistake-queries}, \citet{Doyle2011}, and \cite{VikramD16}.
\seed\ queries have been used in cluster recovery \citep{BCLP21-density} and yield exponential savings in non-realizable learning settings \citep{BalcanHanneke12}.
It also easy to see that \seed\ queries are equivalent to \emph{partial equivalence} queries of~\citet{maass1992lower} and to \emph{subset} plus \emph{superset} queries of~\cite{Angluin88}.
To the best of our knowledge, no work combines \lab\ and \seed\ as we do here. 

Little is known about the \seed\ complexity of learning a concept class $\scH$ actively in polynomial time. On the one hand, the inference dimension lower bounds of~\cite{Kane17} are inapplicable, as \seed\ queries are not local. On the other hand the Littlestone dimension of $\scH$ yields an upper bound, but not necessarily an efficient algorithm; in fact, it is well known that (some sub-problem solved by) Halving is hard in general, see~\cite{gonen2013efficient}.
For $k=2$, we can use \seed\ to emulate \emph{equivalence} queries, for which polynomial-time algorithms are known in some special cases.
In particular, the algorithm of~\cite{maass1994fast} could replace our cutting-planes subroutine under an implicit discretization of the space through a grid with step-size $\scO(\nicefrac{\gamma}{m^4})$. However, this gives a polynomial-time algorithm that uses $\scO(m^2 \log \nicefrac{m}{\gamma})$ \seed\ queries, which is $\scO(m)$ times our bound.
Moreover,~\cite{maass1994fast} use \emph{proper} equivalence queries (i.e., the queried concept must be in the class), for which they show a lower bound of $\Omega(m^2\log \nicefrac{m}{\gamma})$. 
Finally, these techniques do not seem to extend to the case $k > 2$.

Our notion of margin strengthens the convex hull margin of~\cite{BCLP21} by requiring $d(\conv(C_j),\conv(C_i)) > \gamma\diam(C_i)$ rather than $d(C_j,C_i) > \gamma\diam(C_i)$. It is not hard to see that the convex hull margin can be arbitrarily smaller than our strong convex hull margin.
Finally, the polytope margin of~\cite{Gottlieb2018} assumes that each class is in the intersection of a finite number of halfspaces with margin. It is easy to see that this condition is strictly stronger than ours.

\section{Upper Bounds}
\label{sec:ub}
This section gives the proofs of Theorem~\ref{thm:upper_binary} and Theorem~\ref{thm:upper_k}. The algorithm behind both theorems has two phases which are described in the next subsections. The case $k > 2$ is essentially the same as for $k=2$, except for an adaptation in the second phase. 

\subsection{The First Phase: Rounding the Classes}
\label{sub:rounding}
The first phase of our algorithms learns what we call an $\alpha$-rounding of $X$.
\begin{definition}
An $\alpha$-rounding of $X$ (w.r.t.\ $h$) is a sequence of pairs $((X_i,E_i))_{i \in [k]}$ where $(X_i)_{i \in [k]}$ is a partition of $X$, and where $E_i$ for $i\in[k]$ is an ellipsoid such that $X_i \subseteq E_i$ and $E_i \le \alpha \conv(C_i)$.
\end{definition}
The idea is that, if $((X_i,E_i))_{i \in [k]}$ is an $\alpha$-rounding of $X$, then $E_i$ gives an approximation of the pseudometric $d_i$ witnessing the strong convex hull margin of $C_i$. Indeed, let $p_i$ be the pseudometric induced by $E_i$, the one such that $E_i=B_{p_i}(\mu_i,1)$ where $\mu_i$ is the center of $E_i$; it is not hard to prove (see Appendix~\ref{apx:ub}):
\begin{lemma}\label{lem:ar}
If $((X_i,E_i))_{i \in [k]}$ is an $\alpha$-rounding of $X$ then $p_i(\conv(X_i \cap C_i), \conv(X_i \cap C_j)) \ge \frac{\gamma}{\alpha}$ for all distinct $i,j \in [k]$.
\end{lemma}
We will use Lemma~\ref{lem:ar} in the second phase. First, we show how to compute an $\alpha$-rounding of $X$ efficiently. We sample points independently and uniformly at random from $X$ until we find $\Theta(m^2)$ points $S_i$ with the same label $i$. As the VC dimension of ellipsoids in $\R^m$ is $\scO(m^2)$, by standard generalization error bounds with constant probability the MVE of $S_i$ contains at least half of $C_i$. We then store that MVE together with the index $i$, remove $S_i$ from $X$, and repeat until $X$ becomes empty. At that point for each $i \in [k]$ we ``merge'' together all points in the MVEs that were computed for class $i$, and compute the MVE of this merged set. We show that this produces an $\alpha$-rounding of $X$ after $\scO(k \log n)$ rounds in expectation.\footnote{What we actually want is, given a finite set $S \subset \R^m$, an ellipsoid $\EL$ such that $\frac{1}{(1+\epsilon)d}\EL \subset \conv(S) \subset \EL$. This can be computed in $\scO(|S|^{3.5} \ln(|S|/\epsilon))$ operations in the real number model of computation, see~\cite{khachiyan1996rounding}. For simplicity however we just assume that we can compute $\EL=\MVE(S)$ in polytime.} The resulting algorithm \AlgoRound\ is listed below; Figure~\ref{fig:rounding} depicts its behaviour on a toy example.

\begin{algorithm2e}[h!]
\caption{\AlgoRound$(X,k)$}
\DontPrintSemicolon
\SetAlgoVlined 
\lFor{$i \in [k]$}{$h_i \leftarrow 0$}
\While{$X \ne \emptyset$}{
    draw points independently u.a.r.\ from $X$ and \lab\ them until for some $i \in [k]$ we draw a (multi)set of $c m^2$ points from $C_i$\;
    $h_i \leftarrow h_i+1$\;
    $S_i^{h_i} \leftarrow$ the sample of $c m^2$ points from $C_i$\;
    $X_i^{h_i} \leftarrow X \cap \MVE(S_i^{h_i})$ \label{line:AR:E_i}\;
    $X \leftarrow X \setminus X_i^{h_i}$ \label{line:AR:update}\;
}
\For{$i \in [k]$}{
    $X_i \leftarrow X_i^1 \cup \ldots \cup X_i^{h_i}$ (set to $\emptyset$ if $h_i=0$) \label{line:AR:l1}\;
    $E_i \leftarrow \MVE(X_i)$ (set to $\emptyset$ if $X_i=\emptyset$) \label{line:AR:l2}
}
\Return{$((X_i,E_i))_{i \in [k]}$\;}
\end{algorithm2e}
\begin{lemma}\label{lem:algoround}
\AlgoRound$(X,k)$ returns an $m^2(m+1)$-rounding of $X$ in time $\poly(n+m)$ using\linebreak $\scO(k^2 m^2 \log n)$ \lab\ queries in expectation.
\end{lemma}
\begin{proof}(\emph{Sketch})
First we show that $E_i \le m^2(m+1) \conv(C_i)$ for all $i \in [k]$. This is trivial if $E_i=\emptyset$, so let $E_i \ne \emptyset$ and let $\ell_i \ge 1$ be the value of $h_i$ at return time. For every $h=1,\ldots,\ell_i$ let $E_i^h=\MVE(S_i^h)$ and let $\mu_i^h$ be the center of $E_i^h$. Using John's theorem one can show that
$\sigma\!\left(E_i,\mu_i,\frac{1}{m}\right) \subseteq \conv \bigcup_{h=1}^{\ell_i} \sigma\big(\conv(S_i^h), \mu_i^h, m\big)$ and $\sigma\big(\conv(S_i^h), \mu_i^h, m\big) \subseteq \sigma(\conv(C_i), \mu, m(m+1))$.
By taking the union over all $h \in [\ell_i]$ we conclude that $\sigma\!\left(E_i,\mu_i,\frac{1}{m}\right) \subseteq \sigma(\conv(C_i), \mu, m(m+1))$, that is, $E_i \le m^2(m\!+\!1) \conv(C_i)$. It is also easy so see that $(X_i)_{i \in [k]}$ is a partition of $X$, hence $((X_i,E_i))_{i \in [k]}$ is an $m^2(m\!+\!1)$-rounding of $X$. 

For the running time, the \textbf{for} loops perform $k \le n$ iterations, and the \textbf{while} loop performs at most $n$ iterations as each iteration strictly decreases the size of $X$. The running time of any iteration is dominated by the computation of $\MVE(S_i)$ or $\MVE(X_i)$, which takes time $\poly(n+m)$, see above. Hence \AlgoRound$(X,k)$ runs in time $\poly(n+m)$. For the query bounds, the \textbf{while} loop makes $\scO(m^2 k)$ \lab\ queries per iteration. By standard generalization bounds, since the VC dimension of ellipsoids in $\R^m$ is $\scO(m^2)$, $E_i^{h}$ contains at least half of $X \cap C_i$ with probability at least $\frac{1}{2}$, and thus the expected number of rounds before $X$ becomes empty is in $\scO(k \lg n)$, see~\cite{BCLP21}. We conclude that \AlgoRound$(X,k)$ uses $\scO(m^2 k^2 \lg n)$ \lab\ queries in expectation.
\end{proof}

\begin{figure}
    \centering
\begin{tikzpicture}[scale=1.4,rotate=35]
\pgfdeclarelayer{background}
\pgfdeclarelayer{foreground}
\pgfsetlayers{background,main,foreground}

\begin{scope}[xshift=5pt]
\coordinate (x0) at (-.9,1);
\end{scope}

\begin{scope}[rotate=10]
\begin{pgfonlayer}{background}
\ellipsebyfoci{fill=blue,fill opacity=.08}{$(-2.37,3.84)$}{$(-.9,.92)$}{1.2};
\end{pgfonlayer}
\ellipsebyfoci{thick,draw=blue}{$(-2.37,3.84)$}{$(-.9,.92)$}{1.2};
\ellipsebyfoci{thick,draw=blue,densely dotted}{$(-.8,1)$}{$(-1.3,1.7)$}{2};
\begin{scope}[shift={(-1,.6)}]
\ellipsebyfoci{thick,draw=blue,densely dotted}{$(-1.1,2.3)$}{$(-.2,2.6)$}{1.8};
\ellipsebyfoci{thick,draw=blue,densely dotted}{$(-1.5,3)$}{$(-1.2,3.2)$}{2.8};
\begin{pgfonlayer}{foreground}
\node[bluepoint] (x0) at (x0) {};
\node[bluepoint] (x1) at (-.15,0) {};
\node[bluepoint] (x2) at (.55,1.1) {};
\node[bluepoint] (x3) at (-.3,1.5) {};
\node[bluepoint] (x4) at (0,1) {};
\node[bluepoint] (x5) at (-.78,.6) {};
\node[bluepoint] (y1) at (.1,2.3) {};
\node[bluepoint] (y2) at (-1.42,2.2) {};
\node[bluepoint] (y3) at (-.3,3.1) {};
\node[bluepoint] (y4) at (-.2,2.2) {};
\node[bluepoint] (z1) at (-1,2.85) {};
\node[bluepoint] (z2) at (-1.5,2.68) {};
\node[bluepoint] (z3) at (-1.5,3.5) {};
\end{pgfonlayer}
\end{scope}
\end{scope}

\begin{scope}[xshift=5pt]
\begin{pgfonlayer}{background}
\ellipsebyfoci{fill=white}{$(-.62,.6)$}{$(1,.65)$}{1.53};
\ellipsebyfoci{thick,draw=black,fill=white,fill opacity=1}{$(-.62,.6)$}{$(1,.65)$}{1.53};
\end{pgfonlayer}
\ellipsebyfoci{fill=white,opacity=.5}{$(-.62,.6)$}{$(1,.65)$}{1.53};
\ellipsebyfoci{thick,draw=black,densely dotted}{$(.1,.1)$}{$(-.8,1)$}{1.3};
\ellipsebyfoci{thick,draw=black,densely dotted}{$(0,0)$}{$(1,1)$}{1.4};
\begin{pgfonlayer}{foreground}
\coordinate (x0) at (-.9,1);
\node[point] (a1) at (-.77,.3) {};
\node[point] (a2) at (0,-.1) {};
\node[point] (a3) at (-.2,1.05) {};
\node[point] (a4) at (-.5,.5) {};
\node[point] (b1) at (.93,0) {};
\node[point] (b2) at (1.05,1.25) {};
\node[point] (b3) at (.7,.9) {};
\node[point] (b4) at (.3,-.3) {};
\node[point] (b5) at (.3,1.17) {};
\end{pgfonlayer}
\end{scope}

\begin{scope}[rotate=-35]
\node (E1) at (-5.05,1) {$E_1$};
\node (E2) at (1.4,1) {$E_2$};
\end{scope}

\end{tikzpicture}
\caption{A toy example in $\R^2$ with $k=2$; black points are in $C_1$, blue points in $C_2$. \AlgoRound$(X,2)$ computes first the ellipsoids $E_2^1,E_2^2$ (dotted black, from left to right), and then the ellipsoids $E_1^1,E_1^2,E_1^3$ (dotted blue, from left to right). Finally it computes $E_1$ (solid blue) and $E_2$ (solid black). $X_1$ and $X_2$ consist of the points in the blue and white areas respectively. Note that $X_2$ contains a point of $C_1$.}
\label{fig:rounding}
\end{figure}
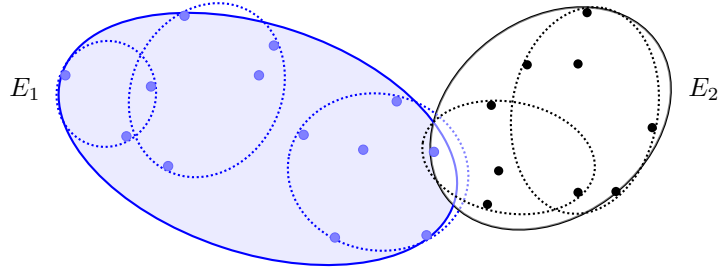

\subsection{The Second Phase: Finding a Separator via Cutting Planes}
\label{sub:cp}
Let $((X_i,E_i))_{i\in[k]}$ be the output of \AlgoRound$(X,k)$, and fix $i \in [k]$. For each $j \in [k] \setminus \{i\}$, we want to separate $X_i \cap C_i$ from $X_i \cap C_j$. To this end, first we use $E_i$ to perform a change of coordinates; this puts $X_i$ inside the unit ball and ensures that $X_i \cap C_i$ and $X_i \cap C_j$ are linearly separated with margin $\gamma_{\text{SVM}}=\Omega(\gamma m^{-3})$. Next, by calling $C_i$ the positive class ($+1$) and $C_j$ the negative class ($-1$), and letting $X=X_i$ for simplicity, one can reduce the task to the following problem. Consider a \emph{partial} classifier $h:X \to \{+1,-1, *\}$. The algorithm has access to an oracle answering queries $\seed(U,y)$ where $U \subseteq X$ and $y \in \{+1,-1\}$, and its goal is to compute a separator of $X$:
\begin{definition}
Let $X \subset \R^m$ and $h:X \to \{+1,-1, *\}$. A \emph{separator} of $X$ (w.r.t.\ $h$) is a partition $(X_+,X_-)$ of $X$ such that, for every $x \in X$, if $h(x)=+1$ then $x \in X_+$ and if $h(x) = -1$ then $x \in X_-$.
\end{definition}
\noindent 
A separator of $X$ can be learned, for instance, by the Perceptron (using \seed\ to find counterexamples). However, this would yield a query and running time bound of $\scO(1/\gamma_{\text{SVM}}^2)=\scO(m^6/\gamma^2)$.
We provide \CuttingPlanes, a cutting-plane algorithm based on \seed\ that is much more query-efficient (in fact, near-optimal):
\begin{theorem}\label{thm:cp}
Let $X \subset \R^m$ and $h:X \to \{+1,-1, *\}$, and suppose $h^{-1}(+1)$ and $h^{-1}(-1)$ are linearly separable with margin $r$. Given $X$ and access to \seed\ for labels $\{+1,-1\}$, \CuttingPlanes$(X)$ computes a separator of $X$ w.r.t.\ $h$ using $\scO(m \log \frac{R}{r})$ \seed\ queries in expectation, where $R=\max_{x\in X} \norm{2}{x}$, and running with high probability\footnote{This means that the running time can be brought in $\poly(m+|X|)$ with probability $1-\exp(-(m+|X|))$.} in time $\poly(m+|X|)$.
\end{theorem}
\begin{proof}(\emph{Sketch})
First, we lift $X$ to $\R^{m+1}$. This reduces the problem to finding a homogeneous linear separator. To this end we let $X'=\{x' : x \in X\}$ where $x'$ is obtained by appending to $x$ an $(m+1)$-th coordinate that is equal to $R$, and we extend $h$ to $X'$ in the obvious way. It is easy to prove that $X'$ has radius at most $2R$ and that in $X'$ the two classes are linearly separable with margin $\frac{r}{2}$.

Next, we learn a separator of $X'$ w.r.t.\ $h$ via cutting planes---see, e.g., \cite{mitchell2003polynomial}. Let $V_0=B^{m+1}(0,1)$. Every point $u \in V_0$ identifies the halfspace $H(u) = \{ z \in \R^{m+1} : \dotp{u,z} \ge 0 \}$. For $i=1,2,\ldots$, $V_i$ will be our version space, and we compute $V_{i+1}$ from $V_i$ as follows. Let $\mu_i$ be the center of mass of $V_i$, and let $X_i' = X' \cap H(\mu_i)$. By issuing \seed$(X_i',-1)$ and \seed$(X' \setminus X_i', +1)$ we learn whether $(X_i',X' \setminus X_i')$ is a separator of $X'$ w.r.t.\ $h$, in which case we return the corresponding partition of $X$, or we obtain a point $u_i$. In the second case, we let $V_{i+1}=V_i \cap U_i$ where $U_i = \{x \in \R^{m+1} : h(u_i) \cdot \dotp{u_i,x} \ge 0\}$. By~\citep[Theorem 2]{Navot2004-BPM} this procedure returns a separator of $X'$ w.r.t.\ $h$ using at most $\frac{2m}{\log\frac{e}{e-1}} \log \frac{4R}{r/2} = \scO\big(m \log \frac{R}{r}\big)$ queries.

Unfortunately, computing $\mu_i$ is hard in general~\citep{Rademacher07Centroid}. We instead compute an estimate $\hat \mu_i$ that, used in place of $\mu_i$, ensures $\frac{\vol(V_{i+1})}{\vol(V_i)}$ is bounded away from $1$ with high probability; the expected query bound follows by adapting the proof of~\citep{Navot2004-BPM}. Assume for the moment that $V_i$ is well-rounded---that is, it contains a ball of radius $r=\poly(m)$ and is contained in a ball of radius $1$. To compute $\hat \mu_i$ we average over $\poly(n+m)$ independent uniform points from $V_i$, which can be draw efficiently thanks to the rounding condition. At this point we use $\hat \mu_i$ in place of $\mu_i$ to invoke \seed\ and obtain a violated constraint $U_i$. Howewer, setting $V_{i+1}=V_i \cap U_i$ could make $V_{i+1}$ far from rounded (too ``thin''), making sampling inefficient at the next round. Therefore we rotate $U_i$ so to obtain a weaker constraint $U_i^*$, one that still contains $V_i \cap U_i$ but that has $\hat \mu_i$ on its boundary, and let $V_{i+1} = V_i \cap U_i^*$. By the assumption on $\hat \mu_i$ this implies that $\vol(V_{i+1}) \ge \frac{1}{3}\vol(V_i)$; therefore by sampling uniform points from $V_i$ we can obtain a large sample in $V_{i+1}$, from which we can put $V_{i+1}$ in a rounding position. See the full proof for all the details.
\end{proof}
To the best of our knowledge, \CuttingPlanes\ is the first efficient algorithm that achieves the query upper bound of Theorem~\ref{thm:cp}, even for the special case of SVM margin.
\subsection{Wrap-Up}
\label{sub:wrapup}
We wrap up our algorithms, starting with the case $k=2$; the case $k \ge 2$ is slightly more involved.
\begin{algorithm2e}[h!]
\caption{\AlgoSC$(X)$}
\DontPrintSemicolon
\SetAlgoVlined 
$((X_1,E_1),(X_2,E_2)) \leftarrow$ \AlgoRound$(X)$\;
\For{$i \leftarrow 1,2$}{
    change system of coordinates so that $E_i$ becomes the unit ball\;
    $(X_{i+}, X_{i-}) \leftarrow$ \CuttingPlanes$(X_i)$ with $h : X_i \to \{1,2\}$ \label{line:cp}\;
}
\Return $(X_{1+}\cup X_{2-}, X_{2+} \cup X_{1-})$
\end{algorithm2e}
\begin{theorem}\label{thm:algo_bin}
Suppose $k=2$. Then \AlgoSC$(X)$ returns $\C=(C_1,C_2)$ in time $\poly(n+m)$ using in expectation $\scO(m^2 \log n)$ \lab\ queries and $\scO(m \log \frac{m}{\gamma})$ \seed\ queries.
\end{theorem}
\begin{proof}
By Lemma~\ref{lem:algoround}, \AlgoRound$(X)$ runs in time $\poly(n+m)$, makes $\scO(m^2 \log n)$ \lab\ queries in expectation, and returns an $\scO(m^3)$-rounding of $X$.
It is immediate to see that, after the change of coordinates, $X_i$ has radius $R \le 1$, while $C_1 \cap X_1$ and $C_2 \cap X_1$ are separated linearly with margin $r=\Omega(\gamma m^{-3})$. By Theorem~\ref{thm:cp} then, \CuttingPlanes$(X_i)$ returns the partition of $X_i$ induced by $h$ in time $\poly(|X_i|+m) = \poly(n+m)$ using $\scO\big(m \log \frac{R}{r}\big) = \scO\big(m \log \frac{m}{\gamma}\big)$ expected \seed\ queries.
\end{proof}

For $k \ge 2$ we proceed as follows. Let $\mathbf{k}=[k]$. We take $X_i$ for each $i \in \mathbf{k}$ in turn, and for each $j \in \mathbf{k} \setminus i$, we use \CuttingPlanes\ to compute a separator for $i,j$ in $X_i$. By intersecting the left side of all those separators we obtain $X_i \cap C_i$. Then we recurse on $X_i \setminus C_i$, updating $\mathbf{k}$ to $\mathbf{k} \setminus i$. The resulting algorithm \AlgoSCK\ is listed below and yields:
\begin{theorem}\label{thm:algo_k}
\AlgoSCK$(X,[k])$ returns $\C$ in time $\poly(n+m)$ using in expectation\linebreak $\scO(k! k^2 \, m^2 \log n)$ \lab\ queries and $\scO\big(k! k^2 \, m \log \frac{m}{\gamma}\big)$ \seed\ queries.
\end{theorem}
\begin{proof}
We adapt the proof of Theorem~\ref{thm:algo_bin}. Observe that \AlgoSCK$(X,[k])$ makes at most $\min(k!,n)$ recursive calls; the $n$ in the $\min$ comes from the fact that any given (recursive) call learns the label of at least one unlabeled point. Now, every (recursive) call makes one invocation to \AlgoRound$(X)$, which by Lemma~\ref{lem:algoround} uses time $\poly(n+m)$ and $\scO(k^2 m^2 \log n)$ \lab\ queries, and $\scO(k^2)$ invocations to \CuttingPlanes$(X_i)$, each of which by Theorem~\ref{thm:cp} uses $\poly(n+m)$ time and $\scO\big(m \log \frac{m}{\gamma}\big)$ \seed\ queries.
\end{proof}

\begin{algorithm2e}[h!] 
\caption{\AlgoSCK$(X,\mathbf{k})$}
\DontPrintSemicolon
\SetAlgoVlined 
$k \leftarrow |\mathbf{k}|$\;
\lIf{$k=1$}{query any point of $X$ and label all of $X$ accordingly}
\Else{
    $((X_i,E_i))_{i \in [k]} \leftarrow$ \AlgoRound$(X)$\;
    \For{$i \in \mathbf{k}$}{
        change system of coordinates so that $E_i$ becomes the unit ball\;
        \For{$j \in \mathbf{k} \setminus i$}{
            $(C_{ij}, \overline{C_{ij}}) \leftarrow \CuttingPlanes(X_i)$ with $h : X_i \to \{i,j\}$\label{line:cpij}\;
        }
        $\widehat C_i \leftarrow \bigcap_{j \in \mathbf{k} \setminus i} {C_{ij}}$\;
        mark all of $\widehat C_i$ with label $i$\;
        \lIf{$X_i \setminus \widehat C_i \ne \emptyset$}{\AlgoSCK$(X_i \setminus \widehat C_i, \mathbf{k} \setminus i)$}\label{line:recur}
    }
}
\end{algorithm2e}

\section{Lower Bounds}
\label{sec:lb}
This section gives a detailed sketch of the proof of Theorem~\ref{thm:lb}, recalled here for convenience:
\lowerbound*

\noindent We first give the sketch for $k=2$, and then extend it to $k \ge 2$. For a full proof see Appendix~\ref{apx:lb}.
\textbf{Set-up.} The construction is adapted from Proposition 2 of~\cite{thiessen2021active}. Let $e_1,\ldots,e_m$ be the canonical basis of $\R^m$ and let $\ell = \left\lfloor \nicefrac{1}{\sqrt{2 \gamma \sqrt{m}}} \right\rfloor$; note that $\gamma \le \frac{m^{-3/2}}{16}$ and $m \ge 2$ ensure $\ell \ge 4$. Let $p=m-1$, and for each $i\in[p]$ and $j \in [\ell]$ define $x_i^j = e_i + j \cdot e_m$. Finally, let $X = \{x_i^j : i \in [p], j \in [\ell]\}$ and define the concept class $\Hyp = \left\{ \bigcup_{i \in [p]} \{x_i^1, \ldots, x_i^{\ell_i}\} : (\ell_1,\ldots,\ell_{p}) \in [\ell]^{p} \right\}$. Let $\C=(C_1,C_2)$ be any partition of $X$ such that $C_1 \in \Hyp$. One can easily verify that $\C$  has strong convex hull margin $\frac{1}{2\ell^2\sqrt{m}}\ge\gamma$. See Figure~\ref{fig:lb} for reference.

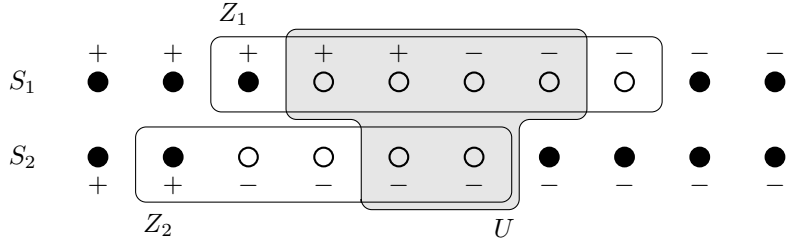
\begin{figure}[h]
    \centering

\begin{tikzpicture}
	[agreement/.style={circle,thick, draw=black,fill=black,inner sep=2.5pt},
	disagreement/.style={circle,thick, draw=black,fill=none,inner sep=2.5pt},
	every edge/.append style={very thick}
	]
	
	\node at (-5.5,0.5) {$S_1$};
	\node[agreement, label=above:{$+$}] at (-4.5,0.5) {};
	\node[agreement, label=above:{$+$}]  at (-3.5,0.5) {};
	\node[agreement, label=above:{$+$}]  at (-2.5,0.5) {};
	\node[disagreement, label=above:{$+$}]  at (-1.5,0.5) {};
	\node[disagreement, label=above:{$+$}]  at (-0.5,0.5) {};
	\node[disagreement, label=above:{$-$}]  at (0.5,0.5) {};
	\node[disagreement, label=above:{$-$}]  at (1.5,0.5) {};
	\node[disagreement, label=above:{$-$}]  at (2.5,0.5) {};
	\node[agreement, label=above:{$-$}]  at (3.5,0.5) {};
	\node[agreement, label=above:{$-$}]  at (4.5,0.5) {};
	\node at (-5.5,-0.5) {$S_2$};
	\node[agreement, label=below:{$+$}]  at (-4.5,-0.5) {};
	\node[agreement, label=below:{$+$}]  at (-3.5,-0.5) {};
	\node[disagreement, label=below:{$-$}]  at (-2.5,-0.5) {};
	\node[disagreement, label=below:{$-$}]  (v1) at (-1.5,-0.5) {};
	\node[disagreement, label=below:{$-$}]  at (-0.5,-0.5) {};
	\node[disagreement, label=below:{$-$}]  at (0.5,-0.5) {};
	\node[agreement, label=below:{$-$}]  at (1.5,-0.5) {};
	\node[agreement, label=below:{$-$}]  at (2.5,-0.5) {};
	\node[agreement, label=below:{$-$}]  at (3.5,-0.5) {};
	\node[agreement, label=below:{$-$}]  at (4.5,-0.5) {};
	\draw[rounded corners] (-3,1.1)--(-3,0.1)--(3,0.1)--(3,1.1)--cycle;
	\node at (-2.7, 1.4) {$Z_1$};
	\draw[rounded corners] (-4,-0.1)--(-4,-1.1)--(1,-1.1)--(1,-0.1)--cycle;
	\node at (-3.7, -1.4) {$Z_2$};
    \draw[rounded corners, fill=black, fill opacity=.1]  (-2,1.2)--(-2,0)--(-1,0)--(-1,-1.2)--(-1,-1.2)--(1.1,-1.2)--(1.1,0)--(2.,0)--(2,1.2)--cycle;
	
	\node at (0.9, -1.45) {$U$};

\end{tikzpicture}
    \caption{$X$ for $p=2$ and $\ell=10$. Filled points represent the agreement region. The maximum point of $S_1\cap C_1$ (resp. $S_2\cap C_1$) can be any point in $Z_1$  (resp. $Z_2$). $U$ is a possible query.}
    \label{fig:lb}
\end{figure}

\textbf{Query bound.} Let $V_0=\{(C_1,C_2) : C_1 \in \Hyp\}$. This is the initial version space. We let the target concept $\C=(C_1,C_2)$ be drawn uniformly at random from $V_0$. Note that for $k=2$, any lower bound on the number of \seed\ queries alone, also holds for any combination of \seed\ and \lab\ queries, as $\lab(x)$ can be simulated by $\seed({x},1)$. Thus, without loss of generality, we can assume that the algorithm is only using \seed\ queries. For all $t=0,1,\ldots$, we denote by $V_t$ the version space after the first $t$ \seed\ queries made by the algorithm. Now fix any $t \ge 1$ and let $\seed(U,y)$ be the $t$-th such query. Without loss of generality we assume $y=1$; a symmetric argument applies to $y=2$. If $U \cap C_1$ contains a point $x$ whose label can be inferred from the first $t-1$ queries, then we return $x$. Therefore we can continue under the assumption that $U$ does not contain any such point (doing otherwise cannot reduce the probability that the algorithm learns nothing). The oracle answers so to maximize $\frac{|V_t|}{|V_{t-1}|}$, as described below.

For each $i \in [p]$ let $S_i=\{x_i^j : j \in [\ell]\}$. We consider $S_i$ as sorted by the index $j$. Let $Z_i$ be the subset of $S_i$ in the disagreement region of $V_{t-1}$ together with the point in $S_i$ preceding this region; observe that this point always exists, as $x^1_i\in C_1$ is in the agreement region. Note that $Z_i$ is necessarily an interval of $S_i$. We let $U_i = Z_i \cap U$ for each $i \in [p]$ and $P(U) = \{i \in [p] : U_i \ne \emptyset \}$. For every $i \in P(U)$, we let $\alpha_i$ be the fraction of points of $Z_i$ that precede the first point in $U_i$. Let $x_i^*=\argmax\{j : x_i^j\in S_i\cap C_1\}$. Observe that $|V_{t-1}|=\prod_{i \in [p]} |Z_i|$.
Indeed, $x_i^*$ is uniformly distributed over $Z_i$; either $x_i^*$ is a point in the disagreement region of $S_i$, or the disagreement region of $S_i$ is fully contained in $C_2$ and $x_i^*$ is the point preceding the disagreement region of $S_i$.
 
Now we show that $\E[|V_{t-1}|/|V_t|] \le m$. Let $\ev$ be the event that $\seed(U,1) = \nil$. Write:
\begin{align}
    \E\left[\frac{|V_{t-1}|}{|V_t|}\right] = \Pr(\ev) \, \E\left[\frac{|V_{t-1}|}{|V_t|} \,\Big|\, \ev\right] + \Pr(\overline \ev) \, \E\left[\frac{|V_{t-1}|}{|V_t|} \,\Big|\, \overline \ev \right] \label{eq:E_Vt}
\end{align}
We bound each one of the two terms in the right-hand side.

For the first term, note that $\ev$ holds if and only if $U_i \cap C_1 = \emptyset$ for all $i \in P(U)$. Since $x_i^*$ is uniformly distributed over $Z_i$, for all $i \in P(U)$ we have $\Pr(C_1 \cap U_i = \emptyset) = \alpha_i$, and since the distributions of those points are independent, then $\Pr(\ev) = \prod_{i \in P(U)} \alpha_i$.
If $\Pr(\ev) > 0$ and $\ev$ holds, then $x_i^*$ is uniformly distributed over the first $\alpha_i |Z_i|$ points of $Z_i$, as the rest of $Z_i$ belongs to $C_2$. This holds independently for all $i$, thus:
\begin{align}
    |V_t| = \Bigg(\prod_{i \in P(U)} \alpha_i |Z_i|\Bigg)  \Bigg(\prod_{i \in [p] \setminus P(U)} |Z_i|\Bigg) = \Bigg(\prod_{i \in P(U)} \alpha_i\Bigg) \Bigg(\prod_{i \in [p]} |Z_i|\Bigg) = |V_{t-1}| \prod_{i \in P(U)} \alpha_i
\end{align}
It follows that $\Pr(\ev) \E\left[\frac{|V_{t-1}|}{|V_t|} \,\Big|\, \ev\right] \leq 1$.

Let us turn to the second term. If $\ev$ does not hold, then $\seed(U,1)$ returns the smallest point $x \in U_i$ for any $i \in P(U)$ such that $C_1 \cap U_i \ne \emptyset$ (note that necessarily $x \in C_1$).
For any fixed $i \in P(U)$, the probability of returning the smallest point of $U_i$ is bounded by $\Pr(C_1 \cap U_i \ne \emptyset)$, which is $1-\alpha_i$; and if this is the case, then we have $|V_t|=(1-\alpha_i)|V_{t-1}|$. Thus:
\begin{align}
    \Pr(\overline \ev) \E\left[\frac{|V_{t-1}|}{|V_t|} \,\Big|\, \overline \ev \right] \le \Pr(\overline \ev) \max_{i \in P(U)} (1-\alpha_i)\frac{1}{(1-\alpha_i)} = \Pr(\overline \ev) \le 1
\end{align}
So the two terms of~\eqref{eq:E_Vt} are both bounded by $1$; we conclude that $\E\left[\frac{|V_{t-1}|}{|V_t|}\right] \le 2$.

Next, fix any $\bar t \ge 1$ and let $\log = \log_2$. By the concavity of $\log$ and by Jensen's inequality:
\begin{align}
    \E\left[\log\frac{|V_0|}{|V_{\bar t}|}\right]
    &= \E\left[\sum_{t=1}^{\bar t}\log\frac{|V_{t-1}|}{|V_t|}\right]
    = \sum_{t=1}^{\bar t}\E\left[\log\frac{|V_{t-1}|}{|V_t|}\right]
    \le \sum_{t=1}^{\bar t}\log\E\left[\frac{|V_{t-1}|}{|V_t|}\right]
\end{align}
Since $\E\left[\frac{|V_{t-1}|}{|V_t|}\right] \le 2$, the right-hand side is at most ${\bar t}$. Now, since $|V_0|=\ell^{p}=\ell^{m-1}$, by Markov's inequality, and since $(m-1) \log \ell - \log 2 \ge \frac{(m-1) \log \ell}{2} \ge \frac{m \log \ell}{4}$:
\begin{align}\label{eq:pr_Vt}
    \Pr(|V_{\bar t}| \le 2) = \Pr\!\left(\log\frac{|V_0|}{|V_{\bar t}|} \ge (m-1) \log \ell - \log 2 \right) \le \frac{4\,\E \!\left[\log\frac{|V_0|}{|V_{\bar t}|}\right]}{m \log \ell} \le \frac{4\, \bar t}{m \log \ell}
\end{align}
Now let $T$ be the random variable counting the number of queries spent by the algorithm, and let $V_T$ be the version space at return time. Since $\C$ is uniform over $V_T$ and $\C$ is returned with probability at least $\frac{1}{2}$, then $\Pr(|V_T| \le 2) \ge \frac{1}{2}$. By~\eqref{eq:pr_Vt} and linearity of expectation,
\begin{align}\label{eq:pr_VT}
    \frac{1}{2} \le \Pr(|V_{T}| \le 2) \le \sum_{\bar t \ge 0} \Pr(T=\bar t) \cdot \frac{4 \bar t}{m \log \ell} = \E[T] \frac{4}{m \log \ell} 
\end{align}
Therefore $\E[T] \ge \frac{m \log \ell} {4}$. Now, since $\ell \ge 4$ then $\ell \ge \frac{4}{5\sqrt{2 \gamma \sqrt{m}}}$, which since $m \le (16\gamma)^{-2/3}$ yields, after calculations, $\ell \ge \sqrt[3]{\nicefrac{1}{\gamma}}\cdot \frac{4^{4/3}}{5\sqrt{2}} > 0.89 \sqrt[3]{\nicefrac{1}{\gamma}}$.
This shows that $E[T] > \frac{m}{24} \log \frac{1}{2\gamma}$,
concluding the proof for $k=2$.

\textbf{Extension to k $\ge$ 2.} For each $s\in \left\lfloor\frac{k}{2}\right\rfloor$ and each pair of classes $C_{2s-1},C_{2s}$, use the construction above shifted along the $m$-th dimension by $(s-1)\ell$. One can easily verify that learning $\C$ is as hard as learning $\left\lfloor\frac{k}{2}\right\rfloor$ independent binary classifiers, for each of which the bound above holds. 

\bibliography{biblio}
\bibliographystyle{plainnat}

\appendix

\section{Appendix for Section~\ref{sec:ub}}\label{apx:ub}

\subsection{Proof of Lemma~\ref{lem:ar}}
If $\mu_i$ is the center of $E_i$, then $E_i = B_{p_i}(\mu_i,1)$. Let $d_i$ be any pseudometric witnessing that $C_i$ has strong convex hull margin $\gamma > 0$. As the margin is invariant under scaling, we can assume $\diam_{d_i}(C_i)=1$ and $\conv(C_i) \subseteq B_{d_i}(z_i,1)$ for some $z_i \in \R^m$. Therefore:
\begin{align}
    B_{p_i}(\mu_i,1) = E_i \le \alpha \conv( C_i) \subseteq \alpha B_{d_i}(z_i,1)
\end{align}
As $p_i$ and $d_i$ are homogeneous and invariant under translation this implies $p_i \ge \frac{d_i}{\alpha}$ and thus $ p_i(\conv(X_i \cap C_j), \conv(X_i \cap C_i)) \ge \frac{1}{\alpha} d_i(\conv(X_i \cap C_j), \conv(X_i \cap C_i))$. Moreover, by monotonicity under taking subsets and by the margin assumption $d_i(\conv(X_i \cap C_j), \conv(X_i \cap C_i)) \ge d_i(\conv(C_j), \conv(C_i)) \ge \gamma \diam_{d_i}(C_i) = \gamma$. Combining the two inequalities yields the thesis.

\subsection{Claim 1}
\begin{claim}\label{lem:affine_union}
Let $K \subset \R^m$ be a convex body, let $E \supseteq K$ be any enclosing ellipsoid, and let $\mu_{E}$ be the centroid of $E$. Let $f(x)=A x + \mu$ be an affine transformation with $\norm{2}{A} \le \lambda$ and $\mu \in K$. Then for any $x \in K$ we have $f(x) \in \sigma(E,\mu_E,\lambda+1)$.
\end{claim}
\begin{proof}
Without loss of generality, we can assume $K$ to be full rank. We can also assume $E$ to be the $\ell_2$ unit ball; otherwise, just apply an appropriate affine transformation at the beginning of the proof, and its inverse at the end. Under these assumptions, for all $x \in K$ we have $\norm{2}{x} \le 1$, and since $\norm{2}{\mu} \le 1$ as well, we obtain:
\begin{align}   
    \norm{2}{f(x)}^2 &= \norm{2}{A x}^2 + \norm{2}{\mu}^2 + 2 \dotp{A x, \mu} \le \lambda^2 + 1 + 2 \lambda = (\lambda+1)^2
\end{align}
which implies $f(x) \in (\lambda+1)E$.
\end{proof}

\subsection{Proof of Lemma~\ref{lem:algoround}}
First, we prove that $E_i \le m^2(m+1) \conv(C_i)$ for all $i \in [k]$. This is trivial if $E_i=\emptyset$, so assume $E_i \ne \emptyset$ and let $\ell_i \ge 1$ be the value of $h_i$ at return time. For every $h=1,\ldots,\ell_i$ let $E_i^h=\MVE(S_i^h)$ and let $\mu_i^h$ be the center of $E_i^h$. If $\mu_i$ is the center of $E_i$
then by John's theorem $\sigma\!\left(E_i,\mu_i,\frac{1}{m}\right) \subseteq \conv(X_i)$, and since $X_i \subset \bigcup_{h=1}^{\ell_i} E_i^h$, then $\conv(X_i) \subseteq \conv\left(\bigcup_{h=1}^{\ell_i} E_i^h\right)$. Moreover $E_i^h \subseteq \sigma\big(\conv(S_i^h), \mu_i^h, m\big)$ for all  $h \in [\ell_i]$, which yields:
\begin{align} \label{eq:sigma_EC}
    \sigma\!\left(E_i,\mu_i,\frac{1}{m}\right)
    \subseteq \conv \bigcup_{h=1}^{\ell_i} \sigma\big(\conv(S_i^h), \mu_i^h, m\big)
\end{align}
Thus we need only to show that the right-hand side is in $\sigma(\conv(C_i), \mu, m(m+1))$ for some $\mu \in \R$.

Let $S_i = \cup_{h=1}^{\ell_i} S_i^h$, let $E = \MVE(S_i)$, and let $\mu$ be the center of $E$. (Note that in general $E \ne E_i$). For every $h \in [\ell_i]$, by applying Claim~\ref{lem:affine_union} from Appendix~\ref{apx:ub} to $f(x) = \sigma(x, \mu_i^h, m)$ and by John's theorem:
\begin{align}
    \sigma\big(\conv(S_i^h), \mu_i^h, m\big) \subseteq \sigma(E, \mu, m+1)
    \subseteq \sigma(\conv(S_i), \mu, m(m+1))
\end{align}
By taking the union over all $h \in [\ell_i]$, and since $\conv(S_i) \subseteq \conv(C_i)$, we obtain:
\begin{align}\label{eq:cup_h_sigma}
    \bigcup_{h=1}^{\ell_i} \sigma\big(\conv(S_i^h), \mu_i^h, m\big) \subseteq \sigma(\conv(C_i), \mu, m(m+1))
\end{align}
As the right-hand side is a convex set,~\eqref{eq:cup_h_sigma} still holds if the left-hand side is replaced by its own convex hull; but that convex hull is the right-hand side of~\eqref{eq:sigma_EC}, which proves the sought claim.

We conclude the proof. For the correctness, since $E_i \le m^2(m\!+\!1) \conv(C_i)$, and since the updates at lines~\ref{line:AR:E_i} and~\ref{line:AR:update} guarantee that $(X_i)_{i \in [k]}$ is a partition of $X$, then $((X_i,E_i))_{i \in [k]}$ is an $m^2(m\!+\!1)$-rounding of $X$. 
For the running time, the \textbf{for} loops perform $k \le n$ iterations, and the \textbf{while} loop performs at most $n$ iterations as each iteration strictly decreases the size of $X$. The running time of any iteration is dominated by the computation of $\MVE(S_i)$ or $\MVE(X_i)$, which takes time $\poly(n+m)$, see above. Hence \AlgoRound$(X,k)$ runs in time $\poly(n+m)$. For the query bounds, the \textbf{while} loop makes $\scO(m^2 k)$ \lab\ queries per iteration. By standard generalization bounds, since the VC dimension of ellipsoids in $\R^m$ is $\scO(m^2)$, $E_i^{h}$ contains at least half of $X \cap C_i$ with probability at least $\frac{1}{2}$, and thus the expected number of rounds before $X$ becomes empty is in $\scO(k \lg n)$, see~\cite{BCLP21}. We conclude that \AlgoRound$(X,k)$ uses $\scO(m^2 k^2 \lg n)$ \lab\ queries in expectation.

\subsection{Pseudocode of \CuttingPlanes\ and full proof of Theorem~\ref{thm:cp}}\label{apx:cplearn_code}
\begin{algorithm2e}[h!]
\DontPrintSemicolon
\SetAlgoVlined 
\caption{\CuttingPlanes$(X)$}
\lIf{$\seed(X,+1)=\nil$}{\Return $(\emptyset,X)$ \label{line:cp_retX}}
\lIf{$\seed(X,-1)=\nil$}{\Return $(X,\emptyset)$ \label{line:cp_retX2}}
$R \leftarrow \max_{x \in X} \norm{2}{x}$\;
$X' \leftarrow \{(x,R) : x \in X\}$\;
$i \leftarrow 0, V_0 \leftarrow B(0,1)$ in $\R^{m+1}$\;
\For{$i \leftarrow 0,\ldots,n$}{
    \If{$i=0$}{
        $\hat \mu_i \leftarrow $ any vector in $S^{m-1}(0,1)$\;
	    $X_i' \leftarrow \{x' \in X' : \dotp{\hat \mu_i, x'} \ge 0\}$\;
	    $X_i \leftarrow$ projection of $X_i'$ on $\R^m$\;
	    \If{\seed$(X_i,-1) = \nil$ and \seed$(X \setminus X_i,+1) = \nil$}{
    		\Return $(X_i,X \setminus X_i)$\;
    	}
	    \Else{
		    let $u_i$ be any point returned by either query\;
		    $V_{i+1} \leftarrow V_i \cap \{x' \in \R^{m+1} : h(u_i) \cdot \dotp{u_i,x'} \ge 0\}$\;
	    }
    }
    \Else{
        draw $\poly(m+n)$ points $z_1,\ldots,z_N$ independently uniformly at random from $V_i$\;
        $\hat \mu_i \leftarrow \frac{1}{N} \sum_{j=1}^N z_j$\;
    	$X_i' \leftarrow \{x' \in X' : \dotp{\hat \mu_i, x'} \ge 0\}$\;
    	$X_i \leftarrow$ projection of $X_i'$ on $\R^m$\;
    	\If{\seed$(X_i,-1) = \nil$ and \seed$(X \setminus X_i,+1) = \nil$}{
    		\Return $(X_i,X \setminus X_i)$\;
    	}
    	\Else{
    		let $u_i$ be any point returned by either query\;
    		$u_i^* \leftarrow u_i - z_0 \cdot \frac{\dotp{u_i,\hat \mu_i}}{\dotp{z_0,\hat\mu_i}}$ where $z_0=h(u_0)\cdot u_0$\;
    		$V_{i+1} \leftarrow V_i \cap \{x' \in \R^{m+1} : h(u_i) \cdot \dotp{u_i^*,x'} \ge 0\}$\;
    	}
}
        draw points independently uniformly at random from $V_i$ until $N=\poly(m+n)$ of them, $z_1,\ldots,z_N$, fall in $V_{i+1}$\;
        use the covariance matrix of $\{z_1,\ldots,z_N\} \cap V_{i+1}$ to compute a coordinate system under which $V_{i+1}$ is $t$-rounded\;
        }
\end{algorithm2e}

We describe how to construct \CuttingPlanes\ step by step. To begin, we issue \seed($X,+1$) and \seed($X,-1$), and if either one returns \nil\ then we immediately return $(\emptyset,X)$ or $(X,\emptyset)$ accordingly.
Otherwise, for the sake of the proof we suppose $h^{-1}(*) = \emptyset$. It is immediate to verify that the algorithm works as it is when $h^{-1}(*) \ne \emptyset$, too, since \seed\ never returns points in $h^{-1}(*)$ and thus, as far as our algorithm is concerned, it behaves identically on $X$ and on $X \setminus h^{-1}(*)$. 

\textbf{Lifting and reduction to the homogeneous case.}
For any $z \in \R^m$ and any $c \in \R$, let $(z,c) \in \R^{m+1}$ be the vector obtained by extending $z$ with a coordinate equal to $c$. For each $x \in X$ let $x'=(x,R)$, and let $X' = \{x' : x \in X\}$. Extend $h$ to $X'$ in the natural way by defining $h(x')=h(x)$ for any $x' \in X'$. We claim that $\{x' \in X' : h(x')=+1\}$ and $\{x' \in X' : h(x')=-1\}$ are separated by a homogeneous hyperplane with margin $\frac{r}{2}$. To see this, let $u \in S^{m-1}$ and $b \in \R$ such that $h(x) \cdot (\dotp{x,u} + b) \ge r$ for all $x \in X$ with $h(x) \ne *$; such $u$ and $b$ exist by the assumptions of the theorem, and note that $b \le R$.
Now let $v=(u,\nicefrac{b}{R})$ and let $u' = \frac{v}{\norm{2}{v}}$; note that $\norm{2}{v} \le \norm{2}{u}+\frac{b}{R} \le 2$. Then, for every $x' \in X'$:
\begin{align}
    \dotp{x',u'} = \frac{\dotp{x',v}}{\norm{2}{v}} = \frac{\dotp{x,u} + R \cdot \nicefrac{b}{R}}{\norm{2}{v}} = \frac{\dotp{x,u} + b}{\norm{2}{v}} 
\end{align}
which implies:
\begin{align}
    h(x') \cdot \dotp{x',u'} =  \frac{h(x) \cdot (\dotp{x,u} + b)}{\norm{2}{v}} \ge \frac{r}{\norm{2}{v}} \ge \frac{r}{2} \label{eq:lifted_margin}
\end{align}
Thus we have reduced the original problem to computing a homogeneous separator with margin.

\textbf{Cutting planes.}
To learn a homogeneous separator we use cutting planes---see, e.g., \cite{mitchell2003polynomial}. Let $V_0$ be the $(m+1)$-dimensional unit ball, which we denote by $B(0,1)$. For all $i \ge 0$ define $V_{i+1}$ as follows. Let $\mu_i$ be the center of mass of $V_i$, let
\begin{align}\label{eq:H_i}
	H_i = \{ x' \in \R^{m+1} : \dotp{\mu_i,x'} \ge 0 \}
\end{align}
and let $X_i' = X' \cap H_i$. Execute \seed$(X_i',-1)$ and \seed$(X' \setminus X_i', +1)$. If both return \nil\ then return $(X_i, X \setminus X_i)$ where $X_i$ is the projection of $X_i'$ on $\R^m$. If either one returns a point $u_i$, let
\begin{align}\label{eq:Z_i}
	Z_i = \{x' \in \R^{m+1} : h(u_i) \cdot \dotp{u_i,x'} \ge 0\}
\end{align}
and let $V_{i+1} = V_i \cap Z_i$. By~\citep[Theorem 2]{Navot2004-BPM}, this procedure returns a separator of $X$ w.r.t.\ $h$ by making at most $\frac{2m}{\log\frac{e}{e-1}} \log \frac{4R}{r/2} = \scO(m \log \frac{R}{r})$ queries. Unfortunately, computing $\mu_i$ is hard in general. Instead we compute a point $\hat \mu_i$ that with high probability has Tukey depth in $[c_1,c_2]$ for some universal $0 < c_1 \le c_2 < 1$. This means that every halfspace $P \subseteq \R^{m+1}$ having $\hat \mu_i$ on its boundary satisfies $c_1 \vol(V_i) \le \vol(P \cap V_i) \le c_2 \vol(V_i)$. By letting $V_{i+1}=V_i \cap P$ for some such $P$, with high the volume of the version space decreases by a factor $c_2$ at every round; by the proof of \citep[Theorem 2]{Navot2004-BPM} this implies that we find our separator of $X$ w.r.t.\ $h$ using $\scO(m \log \nicefrac{R}{r})$ queries in expectation. But moreover the volume of the version space decreases by no more than $c_1$ at every round --- which is crucial, as we describe next.

\textbf{Keeping the version space rounded.}
Let us outline the strategy of the algorithm. We say a convex body $K \subset \R^{m+1}$ is $t$-rounded if $B(0,t) \subseteq K \subseteq B(0,1)$. For every $i=0,1,\ldots$ we maintain the invariant that with high probability, say $1-n^{-c}$ for some universal constant $c$, $V_i$ is $t$-rounded for $t=\Omega(\nicefrac{1}{m})$. More precisely, at each round $i=0,1,\ldots$ we compute a ``temporary'' coordinate system under which $V_i$ is $t$-rounded. Note that $V_0$ is trivially $t$-rounded under the canonical coordinate system given by the canonical basis of $\R^{m+1}$.

Suppose then that, at the beginning of the $i$-th round, $V_i$ is $t$-rounded under some temporary coordinate system.
Then we can efficiently sample points $\epsilon$-uniformly from $V_i$ using the hit-and-run algorithm of~\cite{lovasz2006hit}. More precisely, if we start hit-and-run from the origin, we obtain an $\epsilon$-uniform sample after $\scO\left(m^5 \ln\frac{m}{\epsilon}\right)$ steps. As every step can be implemented in time polynomial in the representation of $V_i$, see~\cite{BCLP21}, and the representation of $V_i$ has size $\scO(m+n)$ since $i \le n$ and every constraint requires $\scO(m)$ bits, then we can sample a $\eps$-uniform point from $K$ in time $\poly(n,m,\ln\nicefrac{1}{\epsilon})$. As shown in~\cite{BCLP21}, for any $\eta,p > 0$, if we set $N=\Theta(m^2/\eta^2p^2)$ and $\epsilon=\Theta(\eta/m)$ then the average $\hat\mu_i$ of $N$ independent $\eps$-uniform samples from $V_i$ satisfies $\Pr(d(\hat\mu_i,\mu_i) \le \eta \diam(V_i)) \ge 1-p$ where $\diam(V_i)$ is the Euclidean diameter of $V_i$. Since $\diam(V_i) \le 2$ as $V_i$ is $t$-rounded, choosing $\eta=1/2m^2$ yields $\Pr(d(\hat\mu_i,\mu_i) \le 1/m^2) \ge 1-p$. We then let $p=n^{-c}/2$ and condition on the good event that $d(\hat\mu_i,\mu_i) \le 1/m^2$. It is not hard to see that any halfspace $P$ having $\hat\mu_i$ on its boundary satisfies $\vol(P \cap V_i) \ge \frac{1}{e}(1-\frac{1}{m})^{m+1} \vol(V_i) = \Omega(\vol(V_i))$, which (by taking $\R^m \setminus P$ as well) implies that $\hat \mu_1$ has Tukey depth in $[c_1,c_2]$ for some universal constants $0<c_1\le c_2<1$ as desired.

We then choose a particular halfspace $P$ having $\hat \mu_i$ on the boundary, denoted by $Z_i^*$, as described below. Then, we set $V_{i+1} = V_i \cap Z_i^*$. Finally, we compute the temporary coordinate system under which $V_{i+1}$ is w.h.p.\ $t$-rounded. To this end we draw again points independently and $\eps$-uniformly at random from $V_i$. 
Since $\vol(V_{i+1}) \ge c_1 \vol(V_i)$, any such sample ends in $V_{i+1}$ with probability at least $c_1-\eps$, hence as long as $\eps < c_1/2$, with probability $1-e^{-\Theta(N)}$ to collect $N$ samples $z_1,\ldots,z_N$ in $V_{i+1}$ we need to draw $\Theta(N)$ samples from $V_i$. Moreover, the $N$ samples in $V_{i+1}$ will be $\frac{\eps}{c_1}$-uniform therein.
At this point, from the covariance matrix of $z_1,\ldots,z_N$ we can then compute an affine transformation that with probability $1-n^{-c}/2$ makes $V_{i+1}$ again $t$-rounded for $t=\Omega(\nicefrac{1}{m})$; see for instance~\cite{vempala10}. By a union bound, then, the round yields with probability $1-n^{-c}$ a temporary coordinate system under which $V_{i+1}$ is $t$-rounded.

We now discuss how to choose $Z_i^*$.

\textbf{Cutting the version space.}
Consider round $i$, and suppose we have successfully computed $\hat\mu_i$ as described above. We compute the halfspace $H_i$ defined by $\hat \mu_i$, and we invoke \seed\ as described above, taking care of excluding any point for which we already know the label. Suppose \seed\ returns $u_i$, and consider the homogeneous halfspace $Z_i$ defined by~\eqref{eq:Z_i}. If we set $V_{i+1} = V_i \cap Z_i$ as anticipated above, then $V_{i+1}$ might be very thin along some direction and/or very small in volume. This means $V_{i+1}$ could be very far from being $t$-rounded (because of its thinness), and we could need too many samples to round it again (because of the small volume). Thus setting $V_{i+1} = V_i \cap Z_i$ might invalidate the $t$-rounding invariant.

We bypass this obstacle as follows. Let $z_0 = h(u_0) \cdot u_0$. This is the normal vector associated with $Z_0$, hence $Z_0=\{x \in \R^{m+1} : \dotp{z_0,x} \ge 0\}$. Note that $V_i \subseteq Z_0$, and that $\dotp{z_0,\hat\mu_i}>0$ since $\hat\mu_i$ lies in the interior of $V_i$ (otherwise it would have Tukey depth $0$). We compute:
\begin{align}
    u_i^* = u_i - z_0 \cdot \frac{\dotp{u_i,\hat \mu_i}}{\dotp{z_0,\hat\mu_i}}
\end{align}
Note that $u_i^*$ is well-defined since $\dotp{z_0,\hat\mu_i}>0$ as noted above. Define:
\begin{align}
    Z_i^* = \{x \in \R^{m+1} \,:\, h(u_i) \cdot \dotp{u_i^*,x} \ge 0\}
\end{align}
Note that, for every $x \in \R^{m+1}$, the definition of $u_i^*$ and the linearity of the inner product yield:
\begin{align}\label{eq:dotp_uistar_x}
    \dotp{u_i^*, x} = \dotp{u_i, x} - \dotp{z_0, x} \cdot \frac{\dotp{u_i,\hat \mu_i}}{\dotp{z_0,\hat\mu_i}}
\end{align}
We then set $V_{i+1}=V_i \cap Z_i^*$. 
 
Now we make two crucial claims. The first one is that $V_i \cap Z_i \subseteq V_i \cap Z_i^*$. In fact, we claim that $Z_0 \cap Z_i \subseteq Z_0 \cap Z_i^*$, which suffices since $V_i \subseteq Z_0$. Let $x \in Z_0 \cap Z_i$. Then:
\begin{align} \label{eq:hui_star}
    h(u_i) \cdot \dotp{u_i^*, x} = h(u_i) \cdot \dotp{u_i, x} - h(u_i) \cdot \dotp{z_0, x} \cdot \frac{\dotp{u_i,\hat \mu_i}}{\dotp{z_0,\hat\mu_i}}
\end{align}
Let us examine the terms of~\eqref{eq:hui_star}. First, $h(u_i) \cdot \dotp{u_i, x} \ge 0$ since $x \in Z_i$. Second, $\dotp{z_0,x} \ge 0$ since $x \in Z_0$. Third, $\dotp{z_0,\hat{\mu}_i} > 0$ as noted above. Thus the term $-h(u_i) \cdot \dotp{z_0, x} \cdot \frac{\dotp{u_i,\hat \mu_i}}{\dotp{z_0,\hat{\mu}_i}}$ has the same sign as $-h(u_i) \cdot \dotp{u_i,\hat \mu_i}$. However, by definition $u_i$ is a counterexample to the labeling given by $H_i$, which means $h(u_i) \cdot \dotp{u_i,\hat \mu_i} < 0$. Therefore $h(u_i) \cdot \dotp{u_i^*, x} \ge 0$, which implies $x \in Z_i^*$ as desired. Therefore the target hypothesis is contained in $V_i \cap Z_i^*$, i.e., in $V_{i+1}$. This ensures that the algorithm is correct as if we used $Z_i$.
The second claim is that $\hat \mu_i$ lies inside $Z_i^*$, and in fact on its boundary, as desired. To this end just substitute $x=\hat\mu_i$ in~\eqref{eq:dotp_uistar_x} to see that $\dotp{u_i^*, \hat \mu_i}=0$.

\textbf{Wrap-up.}
First, observe that the algorithm makes at most $n$ rounds. Indeed, every round either returns (if the \seed\ queries return \nil) or learns the label of some point. Note also that, after having learned the label $h(u_i)$ of the counterexample $u_i$, the version space $V_{i+1}$ may contain hypotheses that label $u_i$ incorrectly. This is because $V_{i+1}$ is obtained from $V_i$ by intersecting with the ``relaxed'' constraint $Z_i^*$ rather than with the constraint $Z_i$ derived from $u_i,h(u_i)$. This however has the only the effect that $u_i$ may be included in the set passed to future \seed\ queries; to avoid this issue, we need only to remove $u_i$ from $X'$ after having learnt its label. Hence, the algorithm makes at most $n$ rounds before $X'$ becomes empty.

Now, at every round the $t$-rounding invariant is maintained with probability $1-n^{-c}$. Therefore, with probability $1-n^{1-c}$ the invariant holds at all rounds. If this is the case, every round takes time $\poly(n+m)$ with probability $1-e^{-\poly(n)}$; see the sampling from $V_{i+1}$ above.
We conclude that with high probability the algorithm has running time $\poly(m+n)$.

\subsection{One-sided margin}
\label{sec:oneside}
We sketch the proof of Theorem~\ref{thm:oneside}. Let $d$ be a metric over $\R^m$ induced by some norm $\norm{d}{\cdot}$. We say $C \subseteq X$ has one-sided strong convex hull margin $\gamma$ with respect to $d$ if $d(\conv(X \setminus C),\conv(C)) \ge \gamma \diam_d(C)$. 

The idea behind Theorem~\ref{thm:oneside} is to compute a Euclidean \emph{one-sided $\alpha$-rounding} of $X$ w.r.t.\ $h$, that is, a set $\widehat X \subseteq X$ such that $C \subseteq \widehat X$ and $\widehat X \le \alpha \conv(C)$, where $C = h^{-1}(+1)$. We will compute $\widehat X$ for $\alpha = \poly\big(\frac{\kappa_d}{\gamma}\big)$, and then use the cutting-planes algorithm of Section~\ref{sub:cp}.
As the margin is invariant under scaling, assume without loss of generality $\inf_{u \in S^{m-1}} \norm{d}{u}=1$ and $\sup_{v \in S^{m-1}}\norm{d}{v}=\kappa_d$.
Let $x=\seed(X,+1)$. If $x=\nil$ then clearly $h=-1$. Otherwise we run \BallSearch$(X,x)$, listed below.
\BallSearch\ sorts $X$ by distance from $x$, and then uses \lab\ queries to perform a binary search and find a pair of points $x_{\lo} \in C$ and $x_{\hi} \in X \setminus C$ adjacent in the ordering. (This works even if the order is not monotone w.r.t.\ the labels). At this point \BallSearch\ guesses a value $t$ for $\frac{\gamma}{\kappa_d}$, starting with $t = 1$.
Given $t$, with a \seed\ query \BallSearch\ checks if there are points of $C$ among the points at distance between $d_{\euc}(x,x_{\hi})$ and $\frac{1}{t}d_{\euc}(x,x_{\hi})$ from $x_{\hi}$. If not, then it lets $\widehat X = X \cap B(x,d_{\euc}(x,x_{\lo}))$, else it lets $\widehat X = X \cap B(x,\frac{1}{t}d_{\euc}(x,x_{\hi}))$. Finally, it checks whether $C \subseteq \widehat X$; if yes then it returns $\widehat X$, else it halves $t$ and repeat. One can show that this procedure stops with $t \ge \frac{\gamma}{2\kappa_d}$, yielding a $\widehat X$ such that $\diam(\widehat X)=\scO(\diam(C)/t)$ and that $C$ and $\widehat X \setminus C$ are linearly separated with margin $\Omega\big(t\frac{\gamma}{\kappa_d}\diam(\widehat X)\big)$. Setting $R=\diam(\widehat X)$ and $r=d_{\euc}(C,\widehat X\setminus C)$, we conclude that $\frac{R}{r}= \poly\big(\frac{\kappa_d}{\gamma}\big)$. At this point by Theorem~\ref{thm:cp} we can compute $C$ by running \CuttingPlanes$(\widehat X)$, which takes time $\poly(n+m)$ and uses $\scO\big(m \log \frac{\kappa_d}{\gamma}\big)$ \seed\ queries in expectation.

\begin{algorithm2e}[h!]
\DontPrintSemicolon
\SetAlgoVlined 
\caption{\BallSearch$(X,x_1)$}
let $x_1,\ldots,x_n$ be the points of $X$ in order of Euclidean distance from $x_1$ (break ties arbitrarily)\;
\lIf{$\lab(x_n)=+1$}{\Return $X$ \label{line:retX}}
$\lo \leftarrow 1$, $\hi \leftarrow n$\;
\While{$\hi - \lo \ge 2$}{ %
    $i \leftarrow \left\lceil \frac{\hi + \lo}{2}\right\rceil$\;
    \lIfElse{$\lab(x_i)=1$}{$\lo \leftarrow i$}{$\hi \leftarrow i$}
}
$t \leftarrow 1$,~~$r \leftarrow d_{\euc}(x_1,x_{\lo})$,~~$R \leftarrow d_{\euc}(x_1,x_{\hi})$\;
\Repeat{$\seed(X \setminus \widehat X, +1) = \nil$}{
$U_i \leftarrow \left\{ x \in X :  R \le d_{\euc}(x,x_1) \le \frac{1}{t} R\right\}$\;
\lIfElse{$\seed(U_i,+1) = \nil$}{$\widehat X \leftarrow X \cap B(x_1, r)$ \label{line:ret_Br}}{$\widehat X \leftarrow X \cap B\!\left(x_1, \frac{1}{t} R\right)$ \label{line:ret_BR}}
$t \leftarrow t/2$\;
}
\Return $\widehat X$;
\end{algorithm2e}

\textbf{A remark on Theorem~\ref{thm:oneside}.} Given two pseudometrics $d$ and $q$ induced by seminorms $\norm{d}{\cdot}$ and $\norm{q}{\cdot}$, let $\kappa_d(q) = \sup_{u \in S_{q}^{m-1}} \norm{d}{u} / \inf_{v \in S_{q}^{m-1}} \norm{d}{v}$. If one can compute $\norm{q}{\cdot}$ efficiently, then Theorem~\ref{thm:oneside} holds with $\kappa_{d}(q)$ in place of $\kappa_{d}$. In fact, Theorem~\ref{thm:oneside} is just the special case where $q = d_{\euc}$. Therefore one can restate Theorem~\ref{thm:oneside} so that $d$ is an arbitrary pseudometric (thus including the case $\kappa_d = \infty$), provided one has access to an approximation $q$ of $d$ with finite distortion.

\section{Appendix for Secion~\ref{sec:lb}}
\label{apx:lb}

\subsection{Full proof of Theorem~\ref{thm:lb}}
\textbf{Construction.} We first discuss the case $k=2$. Let $e_1,\ldots,e_m$ be the canonical basis of $\R^m$. To ease the notation define $p=m-1$; the input set will span a $p$-dimensional subspace. Define:
\begin{align}
    \ell = \left\lfloor \frac{1}{\sqrt{2 \gamma \sqrt{m}}} \right\rfloor 
\end{align}
Since $\gamma \le \frac{m^{-3/2}}{16}$ and $m \ge 2$,
\begin{align}
    \ell \ge \frac{1}{\sqrt{2 \frac{m^{-3/2}}{16} \sqrt{m}}} = \sqrt{8m} \ge 4
\end{align}
For each $i\in[p]$ and $j \in [\ell]$, let $x_i^j = e_i + j \cdot e_m$. Finally, let $X = \{x_i^j : i \in [p], j \in [\ell]\}$. Define the concept class:
\begin{align}
    \Hyp = \left\{ \bigcup_{i \in [p]} \{x_i^1, \ldots, x_i^{\ell_i}\} \;:\; (\ell_1,\ldots,\ell_{p}) \in [\ell]^{p} \right\}
\end{align}

Let $\C=\{C_1,C_2\}$ be any partition of $X$ with $C_1 \in \Hyp$ and $C_2 = X \setminus C_1$. First, we observe that $C_1$ and $C_2$ are separated by a hyperplane. Let $(\ell_1,\ldots,\ell_{p})$ be the vector defining $C_1$. Then we let:
\begin{align}
    u = (-\ell_1,\ldots,-\ell_{p},1)
\end{align}
Then for any $x_i^j \in X$,
\begin{align}
    \langle u, x_i^j\rangle = -\ell_i + j
\end{align}
which is bounded from above by zero if and only if $j \le \ell_i$, that is, if and only if $x_i^j \in C_1$. Hence $C_1$ and $C_2$ admit a linear separator. Next we prove that, under the Euclidean distance, $C_1$ and $C_2$ have strong convex hull margin $\gamma$. Using the vector $u$ defined above, since every $x_i^j \in C_2$ has $j \ge \ell_i+1$, then $\langle u, x_i^j \rangle \ge 1$. This implies:
\begin{align}
    d(\conv(C_1),\conv(C_2)) \ge \frac{1}{\norm{2}{u}} \ge \frac{1}{\sqrt{p \ell^2 + 1}} \ge \frac{1}{\ell\sqrt{m}} \label{eq:d_C1_C2_apdx}
\end{align}
The diameter of $C_1$ is at most that of $X$, which equals $d(x_1^1,x_2^{\ell}) \le \ell - 1 + \sqrt{2} \le 2 \ell$. Together with~\eqref{eq:d_C1_C2_apdx} and the fact that $\ell \le \frac{1}{\sqrt{2 \gamma \sqrt{m}}}$, this provides:
\begin{align}
    d(\conv(C_1),\conv(C_2)) \ge \frac{1}{2\ell^2\sqrt{m}} \, \diam_d(C_1) \ge \frac{2 \gamma \sqrt{m}}{2\sqrt{m}} \, \diam_d(C_1) = \gamma\, \diam_d(C_1)
\end{align}
The same holds for $C_2$. Hence $\C$ has strong convex hull margin $\gamma$.

\textbf{Query bound.} Let $V_0=\{(C_1,C_2) : C_1 \in \Hyp\}$. This is the initial version space. We let the target concept $\C=(C_1,C_2)$ be drawn uniformly at random from $V_0$. For all $t=0,1,\ldots$, we denote by $V_t$ be the version space after the first $t$ \seed\ queries made by the algorithm. Now fix any $t \ge 1$ and let $\seed(U,y)$ be the $t$-th such query. Without loss of generality we assume $y=1$; a symmetric argument applies to $y=2$. If $U \cap C_1$ contains a point $x$ in the agreement region of $V_{t-1}$, i.e., whose label can be inferred from past queries, then we return $x$. Therefore we can continue under the assumption that $U$ does not contain any such point (doing otherwise cannot reduce the probability that the algorithm learns nothing). The oracle answers so to maximize $\frac{|V_t|}{|V_{t-1}|}$, as described below.

For each $i \in [p]$ let $S_i=\{x_i^j : j \in [\ell]\}$. We consider $S_i$ as a sequence of points sorted by the index $j$. Let $Z_i$ be the subset of $S_i$ in the disagreement region of $V_{t-1}$ together with the point in $S_i$ preceding this region; observe that this point always exists, as $x^1_i\in C_1$ is in the agreement region. Note that $Z_i$ is necessarily an interval of $S_i$. We let $U_i = Z_i \cap U$ for each $i \in [p]$ and $P(U) = \{i \in [p] : U_i \ne \emptyset \}$. For every $i \in P(U)$, we let $\alpha_i$ be the fraction of points of $Z_i$ that precede the first point in $U_i$. Let $x_i^*=\argmax\{j : x_i^j\in S_i\cap C_1\}$. Observe that $|V_{t-1}|=\prod_{i \in [p]} |Z_i|$, as $x_i^*$ can be every point of $Z_i$. Indeed, $x_i^*$ is uniformly distributed over $Z_i$; either $x_i^*$ is a point in the disagreement region of $S_i$, or the disagreement region of $S_i$ is fully contained in $C_2$ and $x_i^*$ is the point preceding the disagreement region of $S_i$.
 
Now we show that $\E[|V_{t-1}|/|V_t|] \le p+1$. Let $\ev$ be the event that $\seed(U,1) = \nil$. Write:
\begin{align}
    \E\left[\frac{|V_{t-1}|}{|V_t|}\right] = \Pr(\ev) \, \E\left[\frac{|V_{t-1}|}{|V_t|} \,\Big|\, \ev\right] + \Pr(\overline \ev) \, \E\left[\frac{|V_{t-1}|}{|V_t|} \,\Big|\, \overline \ev \right] \label{eq:E_Vt_apdx}
\end{align}

We bound the two terms of~\eqref{eq:E_Vt_apdx} starting with the first one. Note that $\ev$ holds if and only if $U_i \cap C_1 = \emptyset$ for all $i \in P(U)$. Since $x_i^*$ is uniformly distributed over $Z_i$, for all $i \in P(U)$ we have:
\begin{align}
    \Pr(C_1 \cap U_i = \emptyset) = \alpha_i
\end{align}
And since the distributions of those points are independent:
\begin{align}
    \Pr(\ev) = \prod_{i \in P(U)} \Pr(C_1 \cap U_i = \emptyset) = \prod_{i \in P(U)} \alpha_i
\end{align}
If $\Pr(\ev) > 0$ and $\ev$ holds, then $x_i^*$ is uniformly distributed over the first $\alpha_i |Z_i|$ points of $Z_i$, as the rest of $Z_i$ belongs to $C_2$. This holds independently for all $i$,  thus:
\begin{align}
    |V_t| = \left(\prod_{i \in P(U)} \alpha_i |Z_i|\right)  \left(\prod_{i \in [p] \setminus P(U)} |Z_i|\right) = \left(\prod_{i \in P(U)} \alpha_i\right) \left(\prod_{i \in [p]} |Z_i|\right) = |V_{t-1}| \prod_{i \in P(U)} \alpha_i
\end{align}
It follows that $\Pr(\ev) \E\left[\frac{|V_{t-1}|}{|V_t|} \,\Big|\, \ev\right] \leq 1$.

Let us now bound the second term of~\eqref{eq:E_Vt_apdx}. If $\ev$ does not hold, then $\seed(U,1)$ returns the smallest point $x \in U_i$ for any $i \in P(U)$ such that $C_1 \cap U_i \ne \emptyset$ (note that necessarily $x \in C_1$). For any fixed $i \in P(U)$, the probability of returning the smallest point of $U_i$ is bounded by $\Pr(C_1 \cap U_i \ne \emptyset)$, which is $1-\alpha_i$; and if this is the case, then we have $|V_t|=(1-\alpha_i)|V_{t-1}|$. Thus:
\begin{align}
    \Pr(\overline \ev) \E\left[\frac{|V_{t-1}|}{|V_t|} \,\Big|\, \overline \ev \right] \le \Pr(\overline \ev) \max_{i \in P(U)} (1-\alpha_i)\frac{1}{(1-\alpha_i)} = \Pr(\overline \ev) \le 1
\end{align}
So the two terms of~\eqref{eq:E_Vt} are both bounded by $1$; we conclude that $\E\left[\frac{|V_{t-1}|}{|V_t|}\right] \le 2$.

We can conclude the query bound. For any $\bar t \ge 1$,
\begin{align}
    \E\left[\log\frac{|V_0|}{|V_{\bar t}|}\right]
    &= \E\left[\sum_{t=1}^{\bar t}\log\frac{|V_{t-1}|}{|V_t|}\right]
    \\ &= \sum_{t=1}^{\bar t}\E\left[\log\frac{|V_{t-1}|}{|V_t|}\right]
    \\ &\le \sum_{t=1}^{\bar t}\log\E\left[\frac{|V_{t-1}|}{|V_t|}\right] && \text{Jensen's inequality}
    \\ &\le \sum_{t=1}^{\bar t} \log 2  && \text{see above}
    \\ &= {\bar t}
\end{align}
Since $|V_0|=\ell^{m-1}$, by Markov's inequality, and since $(m-1) \log \ell - \log 2 \ge \frac{(m-1) \log \ell}{2} \ge \frac{m \log \ell}{4}$:
\begin{align}\label{eq:apx:pr_Vt}
    \Pr(|V_{\bar t}| \le 2) = \Pr\!\left(\log\frac{|V_0|}{|V_{\bar t}|} \ge (m-1) \log \ell - \log 2 \right) \le \frac{4\,\E \!\left[\log\frac{|V_0|}{|V_{\bar t}|}\right]}{m \log \ell} \le \frac{4\, {\bar t}}{m \log \ell}
\end{align}
Now let $T$ be the random variable counting the number of queries spent by the algorithm, and let $V_T$ be the version space at return time. Since $\C$ is uniform over $V_T$ and $\C$ is returned with probability at least $\frac{1}{2}$, then $\Pr(|V_T| \le 2) \ge \frac{1}{2}$. By~\eqref{eq:apx:pr_Vt} and linearity of expectation,
\begin{align}\label{eq:apx:pr_VT}
    \frac{1}{2} \le \Pr(|V_{T}| \le 2)
    = \sum_{\bar t \ge 0} \Pr(T=\bar t) \Pr(|V_{\bar t}| \le 2)
    \le \sum_{\bar t \ge 0} \Pr(T=\bar t) \cdot\frac{4 {\bar t}}{m \log \ell}
    = \E[T] \frac{4}{m \log \ell} 
\end{align}
Therefore $\E[T] \ge \frac{m \log \ell} {8}$. Now, since $\ell \ge 4$ then $\ell \ge \frac{4}{5\sqrt{2 \gamma \sqrt{m}}}$, which since $m \le (16\gamma)^{-2/3}$ yields
\begin{align}
    \ell \ge \frac{4}{5\sqrt{2 \gamma (16\gamma)^{-1/3}}} = \sqrt[3]{\frac{1}{\gamma}}\frac{4}{5\sqrt{2 (16)^{-1/3}}} = \sqrt[3]{\frac{1}{\gamma}}\frac{4 \cdot 4^{1/3}}{5\sqrt{2}}
\end{align}
Since $\frac{4^{4/3}}{5\sqrt{2}} > 0.89$, we conclude that:
\begin{align}
    \E[T] > \frac{m \log \frac{0.89}{\sqrt[3]{\gamma}}}{8 \log m}
    >  
    \frac{m\, \frac{1}{3}\log \frac{1}{2\gamma} }{8 \log m}
    =
    \frac{m \log \frac{1}{2\gamma}}{24 \log m}
\end{align}
which concludes the proof for $k=2$.

\textbf{Multiclass.} For any $k\geq 2$ let $k' = \left\lfloor\frac{k}{2}\right\rfloor$. For each $s\in [k']$ consider the construction for the case $k=2$ shifted along the $m$-th dimension by $(s-1)\ell\cdot e_m$:
\begin{align}
X_s = \left\{x_i^j + (s-1)\ell\cdot e_m: i \in [p], j \in [\ell]\right\}
\end{align}
We let $X^*=\bigcup_{s\in[k']} X_s$, and we define the possible subsets of $X^*$ corresponding to class $C_{2s-1}$ as:
\begin{align}
    \Hyp_s = \left\{ \bigcup_{i \in [p]} \left\{x_i^1 +(s\!-\!1)\ell\cdot e_m,\; \ldots,\; x_i^{\ell_i}+ (s\!-\!1)\ell\cdot e_m\right\} \;:\; (\ell_1,\ldots,\ell_{p}) \in [\ell]^{p} \right\}
\end{align}
Finally, let $\scH$ be the set of all partitions $\C=(C_1,\ldots,C_k)$ of $X^*$ such that $C_{2s-1} \in \Hyp_s$ and $C_{2s} = X_s\setminus C_{2s-1}$ for all $s \in [k']$, and let $C_k=\emptyset$ in case $k$ is odd.
The same arguments of the case $k=2$ prove that any such $\C$ has convex hull margin $\gamma$. Indeed, for adjacent classes $C_i,C_{i+1}$ those arguments prove that the strong convex hull margin is at least $\gamma$; for non-adjacent classes, the margin can only be larger.
The random target concept $\C=(C_1,\dots,C_k)$ is obtained by drawing each $C_{2s-1}$ for $s\in [k']$ uniformly at random from $\mathcal{H}_s$, and letting $C_{2s}=X_s\setminus C_{2s-1}$.

We turn to the bound. Consider a generic query \seed$(U,i)$ issued by the algorithm. Without loss of generality we can assume $U \subseteq C_{2s-1} \cup C_{2s} = X_s$ where $s=\lfloor\frac{i}{2}\rfloor$; indeed, by construction of $\scH$, that query can never return a point in $U \setminus X_s$. 
This shows that learning $\C$ requires solving the $k'$ independent binary instances $X_s$, returning $\C_s=(C_{2s-1},C_{2s})$, for $s \in [k']$. As the probability of returning $\C$ is bounded from above by the minimum over $s \in [k]$ of the probability of returning $\C_s$, the algorithm must make at least $\frac{m}{24} \log\frac{1}{2\gamma}$ queries for each $s \in [k']$, concluding the proof.

\section{Appendix for Section~\ref{sec:oneside}}\label{apx:oneside}

\begin{lemma}\label{lem:ballsearch}
Let $C \subseteq X$ have strong convex hull margin $\gamma \in (0,1]$ w.r.t.\ $d$. For any $x_1 \in C$ \BallSearch$(X,x_1)$ takes time $\poly(n+m)$, uses $\scO(\log n)$ \lab\ queries and $\scO(\log \frac{\kappa_d}{\gamma})$ \seed\ queries, and outputs $\widehat X \subseteq X$ such that
\begin{enumerate}\itemsep0pt
    \item $C \subseteq \widehat X$
    \item $d_{\euc}(\conv(C), \conv(\widehat X \setminus C))\ge \frac{\gamma^2}{4\kappa_d^2} \diam(\widehat X)$
\end{enumerate}
\end{lemma}
\begin{proof}
To begin, observe that $d_{\euc} \le d \le \kappa_d \, d_{\euc}$ implies that the ratio between distances changes by a factor at most $\kappa_d$ between $d_{\euc}$ and $d$. In particular this implies that for any set $\widehat X \subseteq X$:
\begin{align}\label{eq:deuc_convC_convX}
 \frac{d_{\euc}(\conv(C), \conv(\widehat X \setminus C))}{\diam(C)}
 \ge \frac{d(\conv(C),\conv(\widehat X \setminus C))}{\kappa_d \, \diam_d(C)}
\end{align}
We will use this inequality below.

Now, suppose line~\ref{line:retX} of \BallSearch\ returns, so $\widehat X = X$. The running time, the query bounds, and point (1) are straightforward. To prove (2), since $x_1,x_n \in C$ we have:
\begin{align}\label{eq:diamCdiamX}
\diam(C) \ge d_{\euc}(x_1,x_n) \ge \frac{1}{2}\diam(X) = \frac{1}{2}\diam(\widehat X) \ge \frac{\gamma}{2\kappa_d} \diam(\widehat X)
\end{align}
where we used $\diam(X) = \max_{a,b \in X} d_{\euc}(a,b) \le \max_{a,b \in X} (d_{\euc}(a,x_1) + d_{\euc}(x_1,b)) \le 2 d_{\euc}(x_1,x_n)$. Therefore $\diam(\widehat X) \le \frac{2\kappa_d}{\gamma} \, \diam(C)$, which together with~\eqref{eq:deuc_convC_convX} and the margin condition gives:
\begin{align}\label{eq:d_convC_X}
 \frac{d(\conv(C),\conv(\widehat X \setminus C))}{\diam(\widehat X)}
 \ge \frac{d_{\euc}(\conv(C), \conv(\widehat X \setminus C))}{\frac{2\kappa_d}{\gamma} \, \diam(C)}
 \ge \frac{d(\conv(C),\conv(\widehat X \setminus C))}{\frac{2\kappa_d}{\gamma} \kappa_d \, \diam_d(C)}
 \ge \frac{\gamma^2}{2\kappa_d^2}
\end{align}

We turn to the \textbf{repeat} loop. Consider a generic iteration just before the update of $t$. We prove:
\begin{enumerate}\itemsep0pt
    \item[(a)] $d(C, \widehat X \setminus C) \ge \min\!\left(t,\frac{\gamma}{\kappa_d}\right)\frac{\gamma}{2\kappa_d} \diam(\widehat X)$
    \item[(b)] if $t \le \frac{\gamma}{\kappa_d}$ then $C \subseteq \widehat X$
\end{enumerate}

First, suppose $\seed(U_i,+1) = \nil$, in which case $\widehat X = X \cap B(x_1,r)$. To prove (a), observe that $x_1,x_{\lo} \in C$ implies:
\begin{align}
\diam(C) \ge d_{\euc}(x_1,x_{\lo}) = r \ge \frac{1}{2} \diam(\widehat X) \ge \min\left(\frac{t}{2},\frac{\gamma}{2\kappa_d}\right) \diam(\widehat X)
\end{align}
Now use the argument above, but with $1/\min\!\big(\frac{t}{2},\frac{\gamma}{2 \kappa_d}\big)$ in place of $\frac{2\kappa_d}{\gamma}$ in~\eqref{eq:d_convC_X}.
To prove (b), note that $x_1 \in C$ and $x_{\hi} \in X \setminus C$ implies $R = d_{\euc}(x_1,x_{\hi}) \ge d_{\euc}(C, X \setminus C)$. Since $d_{\euc} \le d \le \kappa_d\, d_{\euc}$, and by the margin assumptions,
\begin{align}
    \frac{R}{\diam(C)} \ge \frac{d_{\euc}(C,X \setminus C)}{\diam(C)} \ge \frac{d(C,X \setminus C)}{\kappa_d\,\diam_d(C)} \ge \frac{\gamma}{\kappa_d} \ge \min\left(t,\frac{\gamma}{\kappa_d}\right)
\end{align}
Therefore $\diam(C) \le \max\big(\frac{1}{t},\frac{\kappa_d}{\gamma}\big)R$, which implies $C \subseteq X \cap B\big(x_1, \max\big(\frac{1}{t},\frac{\kappa_d}{\gamma}\big)R\big)$. For $t \le \frac{\kappa_d}{\gamma}$ the right-hand side is $X \cap B(x_1, \frac{1}{t} R)$. Note however that $X \cap B(x_1, \frac{1}{t} R) = (X \cap B(x_1, r)) \cup U_i $ since $x_{\lo},x_{\hi}$ are adjacent in the sorted list. But $\seed(U_i,+1)=\nil$, hence $C \subseteq X \cap B(x_1, r) = \widehat X$. 

Next, suppose $\seed(U_i,+1) = y \ne \nil$, in which case $\widehat X = X \cap B(x_1,\frac{1}{t}R)$.
To prove (a), note that $\diam(C) \ge d(x_1,y) \ge R$, and that $\diam(\widehat X) \le 2 \frac{1}{t} R$. Hence $\diam(C) \ge \frac{t}{2} \diam(\widehat X) \ge \min\big(\frac{t}{2},\frac{\gamma}{2 \kappa_d}\big) \diam(\widehat X)$. Now use again the argument above, but with $1/\min\!\big(\frac{t}{2},\frac{\gamma}{2 \kappa_d}\big)$ in place of $\frac{2\kappa_d}{\gamma}$ in~\eqref{eq:d_convC_X}.
To prove (b), the argument for the case above implies $C \subseteq X \cap B\big(x_1, \max\big(\frac{1}{t},\frac{\kappa_d }{\gamma}\big)R\big)$. If $t \le \frac{\gamma}{\kappa_d}$ then the right-hand side is just $\widehat X$.

To conclude the proof, note that by point (b) above the \textbf{repeat} loop returns in $\scO(\log \frac{\kappa_d}{\gamma})$ iterations. Therefore \BallSearch$(X,x_1)$ uses $\scO(\log n)$ \lab\ queries and $\scO(\log \frac{\kappa_d}{\gamma})$ \seed\ queries.
Finally, note that the running time can be brought to $\poly(n+m)$ by storing the output of all \seed\ queries, and replacing $U_i$ with $U_i \setminus U_i \cap \hat C$ where $\hat C \subset C$ is the subset of points of $C$ known so far. In this way, at each \textbf{repeat} iteration either $\widehat X_i \subseteq C$ or we learn the label of some point of $C$ previously unknown. Therefore \textbf{repeat} makes at most $n$ iterations; it is immediate to see that each iteration takes time $\poly(n+m)$ and thus \BallSearch\ runs in time $\poly(n+m)$ as well.
\end{proof}

\subsection{Proof of Theorem~\ref{thm:oneside}}
Let $x=$ \seed$(X,+1)$. If $x=\nil$ then stop and return $\emptyset$. Otherwise run \BallSearch$(X,x)$ to obtain $\widehat X$. By Lemma~\ref{lem:ballsearch} this takes $\poly(n+m)$ time, $\scO(\log n)$ \lab\ queries, and $\scO(\log \frac{\kappa_d}{\gamma})$ \seed\ queries. By Lemma~\ref{lem:ballsearch} $C \subseteq \widehat X$, and $C$ and $\widehat X \setminus C$ are linearly separated with margin $\frac{\gamma^2}{4 \, \kappa_d^2} \diam(\widehat X)$. Thus $\widehat X$ satisfies the assumptions of Theorem~\ref{thm:cp} with $R/r=\frac{4 \, \kappa_d^2}{\gamma^2}$, and by running \CuttingPlanes$(\widehat X)$ we obtain $C$ in time $\poly(n+m)$ using $\scO(m \log \frac{\kappa_d}{\gamma})$ \seed\ queries in expectation.

\section{Bounds for inputs with bounded bit complexity}\label{apx:bit}
We consider the case where $X$ has bounded bit complexity, distinguishing two widely used cases.

\subsection{Rational coordinates}
Supose $X \subset \Q^m$ and every $x \in X$ can be encoded in $b(x) \le B$ bits as follows~\citep{KV2018}. If $x \in \mathbb{Z}$, then $b(x)= 1+\lceil\log(|x|+1)\rceil$. If $x = \nicefrac{p}{q}\in\mathbb{Q}$ with $p,q\in\mathbb{Z}$ coprime, then $b(x)=b(p)+b(q)$. If $x\in\mathbb{Q}^m$, then $b(x)=m + \sum_{i \in [m]} b(x_i)$. We show that $B$ gives a lower bound on the margin. The argument is related to \citet{kwek1998pac}.
\begin{lemma}\label{lem:bitmargin}
Suppose $X \subset \Q^m$ has bit complexity bounded by $B$, and suppose $C \subseteq X$ and $X \setminus C$ are linearly separable. Then $d(\conv(C),\conv(X \setminus C)) \ge 2^{-\scO(m^2 B)}$.
\end{lemma}
\begin{proof}
Let $P=\conv(C)$ and let $H$ be a hyperplane containing a face of $P$. By Lemma 4.5 of~\cite{KV2018}, $H=\{x \in \mathbb{R}^m : \dotp{w,x} = t\}$ for some $w \in \Q^m$ and $t \in \Q$ such that $b(w)+b(t) \le 75 m^2 B$. 
The distance between $H$ and any $x \in X \setminus C$ is: 
\begin{align}
d(x,H)=\frac{|\!\dotp{w,x}-t|}{\norm{2}{w}}
\end{align}
To bound $|\!\dotp{w,x}-t|$ suppose $w,x,t$ are encoded by:
\begin{align}
    w_i = \frac{p_w^i}{q_w^i} \quad i \in [m], \qquad
    x_i = \frac{p_x^i}{q_x^i} \quad i \in [m], \qquad
    t = \frac{p_t}{q_t}
\end{align}
Replacing those quantities in the expression of $|\!\dotp{w,x}-t|$, taking the common denominator, observing that the numerator of the resulting expression is an integer, and recalling that $|\!\dotp{w,x}-t| > 0$, we deduce:
\begin{align}
    |\!\dotp{w,x}-t| \ge \frac{1}{q_t \prod_{i \in [m]} q_w^i q_x^i}
\end{align}
However, since $b(x) = \scO(\log (1+|x|))$ for any $x \in \Z$,
\begin{align}
b\left(q_t \prod_{i \in [m]} q_w^i q_x^i\right)
&= \scO\left( b(q_t) + \sum_{i \in [m]} (b(w_i) + b(x_i))\right) 
= \scO(b(t)+b(w)+b(x) )
\end{align}
which therefore is in $\scO(m^2 B)$. Therefore $|\!\dotp{w,x}-t| \ge 2^{-\scO(m^2 B)}$.
To bound $\norm{2}{w}$ we just note that $\norm{2}{w} \le \norm{1}{w} \le 2^{b(w)} \le 2^{75 m^2 B}$. We conclude that:
\begin{align}
d(x,H) = \frac{|\!\dotp{w,x}-t|}{\norm{2}{w}} \ge 2^{-\scO(m^2 B)}
\end{align}
The proof is complete.
\end{proof}  

\begin{corollary}\label{cor:rational_X_ub}
Suppose $X\subset \N^m$ has bit complexity bounded by $B \in \N$ in the rational coordinates model, and let $\C=(C_1,\ldots,C_k)$ be a partition of $X$ such that $C_i,C_j$ are linearly separable for every distinct $i,j \in [k]$. Then $\C$ can be learned in time $\poly(n+m)$ using $\scO(k^2 m^3 B)$ \seed\ queries in expectation. 
\end{corollary}
\begin{proof}
Any $x \in X$ satisfies $\norm{2}{x} \le \norm{1}{x} \le 2^B$, and by Lemma~\ref{lem:bitmargin} any two distinct classes $C_i,C_j \in \C$ are linearly separable with margin $r=2^{-\scO(m^2 B)}$. By Theorem~\ref{thm:cp}, \CuttingPlanes$(X)$ with \seed\ restricted to classes $i,j$ returns a separator for $C_i$ and $C_j$ in time $\poly(m+n)$ using $\scO(m \log \frac{R}{r})=\scO(m^3 B)$ \seed\ queries in expectation. By intersecting the separators for all $j \in [k] \setminus i$ we obtain $C_i$. Repeating this process for all $i \in [k]$ yields the claim.
\end{proof}

\subsection{Grid}
Let $c>0$ be such that $1/c$ is an integer and suppose that $X \subseteq Q = \{-1,-1+c,\ldots,1-c,1\}^m$. We call this the grid model. If $1/c \le 2^{B/m}-1$ then we say that the bit complexity of $X$ is bounded by $B$.
\begin{corollary}\label{cor:grid_X_ub}
Suppose $X\subset \N^m$ has bit complexity bounded by $B \in \N$ in the grid model, and let $\C=(C_1,\ldots,C_k)$ be a partition of $X$ such that $C_i,C_j$ are linearly separable for every distinct $i,j \in [k]$. Then $\C$ can be learned in time $\poly(n+m)$ using $\scO(k^2 m(B + \log m))$. \seed\ queries in expectation. 
\end{corollary}
\begin{proof}
We use the approach of~\cite{gonen2013efficient}. Let $c>0$ be such that $1/c$ is an integer and suppose that $X \subseteq Q = \{-1,-1+c,\ldots,1-c,1\}^m$. By Lemma 10 of~\cite{gonen2013efficient}, any two sets in $Q$ that are linearly separable are also linearly separable with margin $r=(c/\sqrt{m})^{m+2}$. We can thus apply \CuttingPlanes\ as in the proof of Corollary~\ref{cor:rational_X_ub}, obtaining for separating every $C_i,C_j$ a running time of $\poly(m+n)$ and an expected query bound of $\scO(m \log \frac{R}{r})=\scO(m^2 \log (m/c))$. Since $c \ge 2^{-B/m}-1$, then the bound becomes $\scO(m^2 \log (m 2^{B/m})) = \scO(m^2 (B/m + \log m)) = \scO(m(B + \log m))$. This proves the total expected query bound of $\scO(k^2m(B + \log m))$.
\end{proof}

\end{document}